%% file: main.tex
\documentclass{article}
\usepackage{microtype}
\usepackage{graphicx}
\graphicspath{{Plots/}}
\usepackage{subfigure}

\usepackage{titletoc}
\usepackage[page,header]{appendix}

\usepackage{booktabs} % for professional tables
\usepackage{hyperref}
\usepackage[accepted]{icml2023}

\usepackage{amssymb}
\usepackage[mathscr]{euscript}
\usepackage{amsthm}
\usepackage{amsmath}
\usepackage{dsfont}
\usepackage{thmtools} 
\usepackage{thm-restate}
\usepackage[capitalize,noabbrev]{cleveref}
\usepackage{subdepth}
\let\originalleft\left
\let\originalright\right
\renewcommand{\left}{\mathopen{}\mathclose\bgroup\originalleft}
\renewcommand{\right}{\aftergroup\egroup\originalright}
\usepackage{etoolbox}  % Required for ifstrempty
\usepackage{enumitem}  % Required for custom itemize/enumerate
\usepackage{float}

\newcommand{\R} {{\mathbb R}}

\newcommand{\pii} {{p}}
\newcommand\trans{^{\mathrm{T}}}
\newcommand\Id{\mathrm{Id}}
\newcommand\EE{\mathbb{E}}
\newcommand\bracket[1]{\left[ #1 \right]}
\newcommand\brackett[1]{[ #1 ]}
\newcommand\expect[2][]{\EE\ifstrempty{#1}{}{_{#1}}\bracket{#2}}
\newcommand\expectt[2][]{\EE\ifstrempty{#1}{}{_{#1}}\brackett{#2}}

\newcommand\norm[1]{\left\lVert #1 \right\rVert}

\newcommand\paren[1]{\left( #1 \right)}

\DeclareMathOperator{\diag}{diag}
\DeclareMathOperator{\Tr}{Tr}

\newtheorem{theorem}{Theorem}[section]
\newtheorem{lemma}{Lemma}[section]
\newtheorem{proposition}{Proposition}[section]
\newtheorem{definition}{Definition}[section]

\newtheorem{assumption}{Assumption}[section]

\begin{document}

\twocolumn[
\icmltitle{Conditionally Strongly Log-Concave Generative Models}


\icmlsetsymbol{equal}{*}

\begin{icmlauthorlist}
\icmlauthor{Florentin Guth}{equal,ens}
\icmlauthor{Etienne Lempereur}{equal,ens}
\icmlauthor{Joan Bruna}{nyu}
\icmlauthor{Stéphane Mallat}{cdffi}
\end{icmlauthorlist}

\icmlaffiliation{ens}{Département d'informatique, \'Ecole Normale Supérieure, Paris, France}
\icmlaffiliation{nyu}{Courant Institute of Mathematical Sciences and Center for Data Science, New York University, USA}
\icmlaffiliation{cdffi}{Collège de France, Paris, France, and Flatiron Institute, New York, USA}

\icmlcorrespondingauthor{Florentin Guth}{florentin.guth@ens.fr}
\icmlcorrespondingauthor{Etienne Lempereur}{etienne.lempereur@ens.fr}

\icmlkeywords{Machine Learning, ICML}

\vskip 0.3in
]



\printAffiliationsAndNotice{\icmlEqualContribution} % otherwise use the standard text.

\begin{abstract}
There is a growing gap between the impressive results of deep image generative models and classical algorithms that offer theoretical guarantees. The former suffer from mode collapse or memorization issues, limiting their application to scientific data. The latter require restrictive assumptions such as log-concavity to escape the curse of dimensionality. We partially bridge this gap by introducing conditionally strongly log-concave (CSLC) models, which factorize the data distribution into a product of conditional probability distributions that are strongly log-concave. This factorization is obtained with orthogonal projectors adapted to the data distribution. It leads to efficient parameter estimation and sampling algorithms, with theoretical guarantees, although the data distribution is not globally log-concave. We show that several challenging multiscale processes are conditionally log-concave using wavelet packet orthogonal projectors. Numerical results are shown for physical fields such as the $\varphi^4$ model and weak lensing convergence maps with higher resolution than in previous works. 
% Additionally, an explicit parameterization of the likelihood can be retrieved from the model. 
% This approach generalizes the wavelet conditional renormalization decomposition of physical fields with faster score-based algorithms, with no cost in statistical performance. 
\end{abstract}

\section{Introduction}

Generative modeling requires the ability to estimate an accurate model of a probability distribution from a training dataset, as well as the ability to efficiently sample from this model. Any such procedure necessarily introduces errors, 
%This procedure introduces modeling errors 
due to limited expressivity of the model class, learning errors of selecting the best model within that class, and sampling errors due to limited computational resources. For high-dimensional data, it is highly challenging to control all errors with polynomial-time algorithms. Overcoming the curse of dimensionality requires exploiting structural properties of the probability distribution. For instance, theoretical guarantees can be obtained with restrictive assumptions of log-concavity, or with low-dimensional parameterized models. In contrast, recent deep-learning-based approaches such as diffusion models \citep{dall-e2,imagen,stable-diffusion} have obtained impressive results for distributions which do not satisfy these assumptions. Unfortunately, in such cases, theoretical guarantees are lacking, and diffusion models have been found to memorize their training data \citep{data-extraction-diffusion-models,data-replication-diffusion-models}, which is inappropriate for scientific applications. The disparity between these two approaches highlights the need for models which combine theoretical guarantees with sufficient expressive power. This paper contributes to this objective by defining the class of conditionally strongly log-concave distributions. We show that it is sufficiently rich to model the probability distributions of complex multiscale physical fields, and that such models can be sampled with fast algorithms with provable guarantees.

\paragraph{Sampling and learning guarantees.}
While the theory for sampling log-concave distributions is well-developed \citep{chewi22}, simultaneous learning and sampling guarantees for general non-log-concave distributions are less common. \citet{block2020fast} establish a fast mixing rate of multiscale Langevin dynamics under a manifold hypothesis.
\citet{koehler2022statistical} studies the asymptotic efficiency of score-matching compared to maximum-likelihood estimation under a global log-Sobolev inequality, which is not quantitative beyond globally log-concave distributions. 
\citet{chen2022sampling,chen2022improved} establish polynomial sampling guarantees for a reverse score-based diffusion, given a sufficiently accurate estimate of the time-dependent score. \citet{sriperumbudur2013density, sutherland2018efficient, domingo2021energy} study density estimation with energy-based models under different infinite-dimensional parametrizations of the energy. They use various metrics including score-matching to establish statistical guarantees that avoid the curse of dimensionality, under strong smoothness or sparsity assumptions of the target distribution.  
Finally, \citet{balasubramanian2022towards} derive sampling guarantees in Fisher divergence of Langevin Monte-Carlo beyond log-concave distributions. While these hold under a general class of target distribution, such Fisher guarantees are much weaker than Kullback-Leibler guarantees. Bridging this gap requires some structural assumptions on the distribution.

\paragraph{Multiscale generative models.}
Images include structures at all scales, and several generative models have relied on decompositions with wavelet transforms \citep{yu2020wavelet,gal2021swagan}. More recently, \citet{marchand_wavelet_2022} established a connection between the renormalization group in physics and a conditional decomposition of the probability distribution of wavelet coefficients across scales. These models rely on maximum likelihood estimations with iterated Metropolis sampling, which leads to a high computational complexity. They have also been used with score matching \citep{guth_wavelet_2022,kadkhodaie-local-conditional-models} in the context of score-based diffusion models \citep{song2020score}, which suffer from memorisation issues. 

\paragraph{Conditionally strongly log-concave distributions.}
We consider probability distributions whose Gibbs energy is dominated by quadratic interactions, 
\begin{equation*}
p(x) = \frac1Z e^{-E(x)} ~~\text{with } E(x) = \frac 1 2 x\trans K x + V(x).
\end{equation*}
The matrix $K$ is positive symmetric and $V$ is a non-quadratic potential. If $V$ is non-convex, then $p$ is a priori not log-concave. However, the Hessian of $E$ may be dominated by the large eigenvalues of $K$, whose corresponding eigenvectors define directions in which $p$ is log-concave. For multiscale stationary distributions, $K$ is a convolution whose eigenvalues have a power-law growth at high frequencies. As a result, the conditional distribution of high frequencies given lower frequencies may be log-concave.

\Cref{sec:section3} introduces factorizations of probability distributions into products of conditional distributions with arbitrary hierarchical projectors. If the projectors are adapted to obtain strongly log-concave factors, we prove that maximum likelihood estimation can be replaced by score matching, which is computationally more efficient. The MALA sampling algorithm also has a fast convergence due to the conditional log-concavity. 
\Cref{sec:section4} describes a class of multiscale physical processes that admit conditionally strongly log-concave (CSLC) decompositions with wavelet packet projections. This class includes the $\varphi^4$ model studied in statistical physics.
These results thus provide an approach to provably avoid the numerical instabilities at phase transitions observed in such models. %, known as ``critical slowing down''.
We then show in \Cref{sec:numerics} that wavelet packet CSLC decompositions provide accurate models of cosmological weak lensing images, synthesized as test data for the Euclid outer-space telescope mission \citep{laureijs2011euclid}.

The main contributions of the paper are:
\begin{itemize}
\item The definition of general CSLC models, which provide learning guarantees by score matching and sampling convergence bounds with MALA.
\item CSLC models of multiscale physical fields using wavelet packet projectors. We show that $\varphi^4$ and weak lensing both satisfy the CSLC property, which leads to efficient and accurate generative modeling.
\end{itemize}
The code to reproduce our numerical experiments is available at \url{https://github.com/Elempereur/WCRG}.

\input{section_cslc.tex}

\input{section_wavpack.tex}

\input{section_numerics.tex}

\input{section_conclusion.tex}

\section*{Acknowledgments}

This work was partially supported by a grant from the PRAIRIE 3IA Institute of the French ANR-19-P3IA-0001 program. We thank Misaki Ozawa for providing the $\varphi^4$ training dataset and his helpful advice on the numerical experiments. We thank the anonymous reviewers and area chair whose feedback have improved the paper significantly.

\bibliography{references}
\bibliographystyle{icml2023}

\input{appendix.tex}

\end{document}

%% file: section_cslc.tex
\section{Conditionally Strongly Log-Concave Models}
\label{sec:section3}

\Cref{sec:orthogonal_factorization} introduces conditionally strongly log-concave models, by factorizing the probability density into conditional probabilities. For these models, 
\Cref{sec:conditional_learning,sec:sm_exponential}  give upper bounds on learning errors with score matching algorithms, and  \Cref{sec:conditional_sampling} on sampling errors with a Metropolis-Adjusted Langevin Algorithm (MALA). Proofs of the mathematical results can be found in \Cref{sec:proofs_sec3}.

\subsection{Conditional Factorization and Log-Concavity}
\label{sec:orthogonal_factorization}

%Learning and sampling errors can be controlled with strong log-concavity properties. This property is often not verified by complex data distributions. However, the weaker property of conditional strong log-concavity (CSLC) introduced below is sufficient. It depends on a factorization of the probability density as a product of conditional probabilities. 
We introduce a probability factorization based on orthogonal projections on progressively smaller-dimensional spaces. The projections are adapted to define strongly log-concave conditional distributions. 

\paragraph{Orthogonal factorization.}
Let $x \in \mathbb{R}^d$. A probability distribution $p(x)$ can be decomposed into a product of autoregressive conditional probabilities
\begin{equation}
    \label{eq:autoregressive}
    p(x) = p(x[1]) \prod_{i=2}^d p(x[i] \,|\, x[1], \dots, x[i-1]).
\end{equation}
However, more general factorizations can be obtained by considering blocks of variables in an orthogonal basis. We initialize the decomposition with $x_0 = x$. For $j =1$ to $J$, we recursively split $x_{j-1}$ in two orthogonal projections:
\begin{equation*}
    x_j = G_j x_{j-1} \text{ and } \bar x_j = \bar G_j x_{j-1},
\end{equation*}
where $G_j$ and $\bar G_j$ are unitary operators such that $G_j\trans G_j + \bar G_j\trans \bar G_j = \Id$. It follows that
\begin{equation}
   \label{eqn:invwav}
    x_{j-1} = G_j\trans x_{j} + \bar{G}_j\trans \bar{x}_{j}.
\end{equation}
Let $d_j = \text{dim}(x_j)$ and  $\bar d_j = \text{dim}(\bar{x}_j)$, then $d_{j-1} = d_{j} + \bar d_j$.

Since the decomposition is orthogonal, for any probability distribution $p$ we have
\begin{equation*}
    p(x_{j-1}) = p(x_j, \bar x_j) = p(x_j) p(\bar{x}_j | x_j).
\end{equation*}
Cascading this decomposition $J$ times gives
\begin{equation}
    \label{condexp}
    p(x)  = p(x_J) \prod_{j=1}^J p(\bar{x}_j |x_j),
\end{equation}
which generalizes the autoregressive factorization (\ref{eq:autoregressive}). The properties of the factors $p(\bar x_j|x_j)$ depend on the choice of the orthogonal projectors $G_j$ and $\bar G_j$, as we shall see below.

\paragraph{Model learning and sampling.}
A parametric model $p_\theta(x)$ of $p(x)$ can be defined from \cref{condexp} by computing parametric models of $p(x_J)$ and each $p(\bar{x}_j |x_j)$:
\begin{equation}
    \label{paramModel}
    p_\theta(x)  = p_{\theta_J} (x_J) \prod_{j=1}^J p_{\bar \theta_j} (\bar{x}_j |x_j),
\end{equation}
with $\theta = (\theta_J , \bar \theta_j)_{j \geq J}$.

Learning this model then amounts to optimizing the parameters $\theta_J, (\bar{\theta}_j)_j$ from available data, so that the resulting distributions are close to the target. We measure the associated learning errors with the Kullback-Leibler divergences $\epsilon^L_J = \mathrm{KL}_{x_J} (p (x_J) \,\|\, p_{\theta_J}(x_J))$ and
\begin{equation*}
\bar \epsilon^L_j = \expect[x_j]{\mathrm{KL}_{\bar x_j} (p (\bar x_j |x_j) \,\|\, p_{\bar \theta_j}(\bar x_j |x_j))},\ j\leq J.
\end{equation*}

Once the parameters have been estimated, we sample from $p_\theta$ as follows. We first compute a sample $x_J$ of $p_{\theta_J}$. The sampling introduces an error, which we measure with $\epsilon^S_J = \mathrm{KL}_{x_J} (\hat p_{\theta_J}(x_J) \,\|\, p_{\theta_J}(x_J))$, where $\hat{p}_{\theta_J}$ is the law of the samples returned by the algorithm. For each $j \leq J$, given the sampled $x_j$, we compute a sample $\bar x_j$ of $p_{\bar \theta_j} (\bar x_j | x_j)$ and recover $x_{j-1}$ with \cref{eqn:invwav}, up to $j=1$, where it computes $x = x_0$. Let $\hat p_{\bar \theta_j}$ be the law of computed samples $\bar{x}_j$. It also introduces an error
\begin{equation*}
    \bar \epsilon^S_j = \expect[x_j]{\mathrm{KL}_{\bar x_j} (\hat p_{\bar \theta_j}(\bar x_j |x_j) \,\|\, p_{\bar \theta_j}(\bar x_j |x_j))},\ j \leq J.
\end{equation*}
Let $\hat p$ be the (joint) law of the computed samples $x$. The following proposition relates the total variation distance $\mathrm{TV}(\hat p, p)$ with the learning and sampling errors for each $j$.

\begin{restatable}[Error decomposition]{proposition}{errordecomp}
    \label{prop:global_error}
    \begin{equation*}
        \mathrm{TV}( \hat{p} , p) \leq \frac{1}{\sqrt{2}}\left(  \sqrt{
        \epsilon^L_J + \sum_{j=1}^J \bar \epsilon^{L}_j }
        + \sqrt{\epsilon^S_J + \sum_{j=1}^J \bar \epsilon^{S}_j}\right).
    \end{equation*}
\end{restatable}

The overall error depends on the sum of learning and sampling errors for each conditional probability distribution. Therefore, to control the total error, we need sufficient conditions ensuring that each of these sources of error is small. We introduce CSLC models for this purpose.

\paragraph{Conditional strong log-concavity.}
We recall that a distribution $p$ is strongly log-concave (SLC) if there exists $\beta[p] \geq \alpha[p] > 0$ such that
\begin{equation}
\alpha[p] \Id \preceq -\nabla_x^2 \log p(x) \preceq \beta[p] \Id,\ \forall x.
\end{equation}
\begin{definition}
    \label{def:marginal_logconcave}
    We say that $p(x)  = p(x_J) \prod_{j=1}^J p(\bar{x}_j |x_j)$ is \emph{conditionally strongly log-concave} (CSLC) if each $p(\bar x_j|x_j)$ is strongly log-concave in $\bar{x}_j$ for all $x_j$.
\end{definition}
Conditional log-concavity is a weaker condition than (joint) log-concavity. If $p(x)$ is log-concave, then it has a convex support. On the other hand, conditional log-concavity only constraints slices (through conditioning) and projections (through marginalization) of the support of $p(x)$. \Cref{fig:logconcave_examples} illustrates that a jointly log-concave distribution is conditionally log-concave (and $p(x_J)$ is furthermore log-concave), but the converse is not true. Conditional log-concavity also depends on the choice of the orthogonal projections $G_j$ and $\bar G_j$ which need to be adapted to the data. A major issue is to identify projectors that define a CSLC decomposition, if it exists. We show in \Cref{sec:section4} that this can be achieved for a class of physical fields with wavelet packet projectors.

\begin{figure}
    \centering
    \null\hfill
    \includegraphics[width=0.475\linewidth]{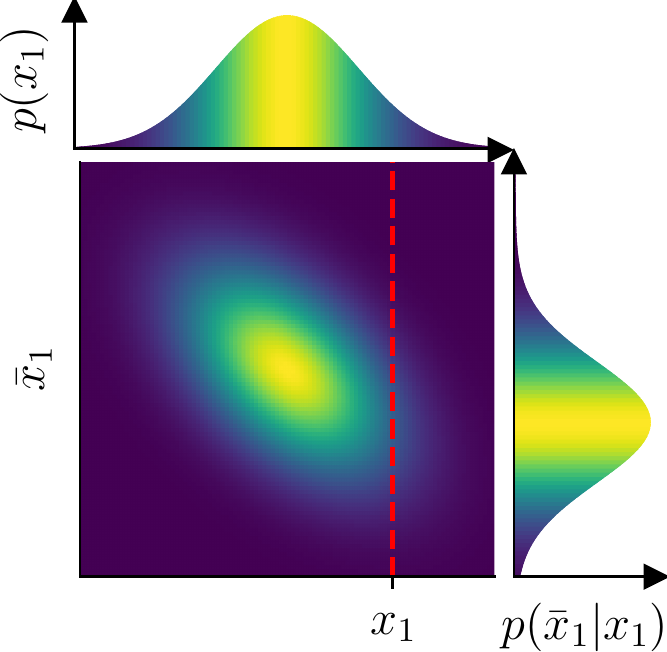}
    \hfill
    \includegraphics[width=0.475\linewidth]{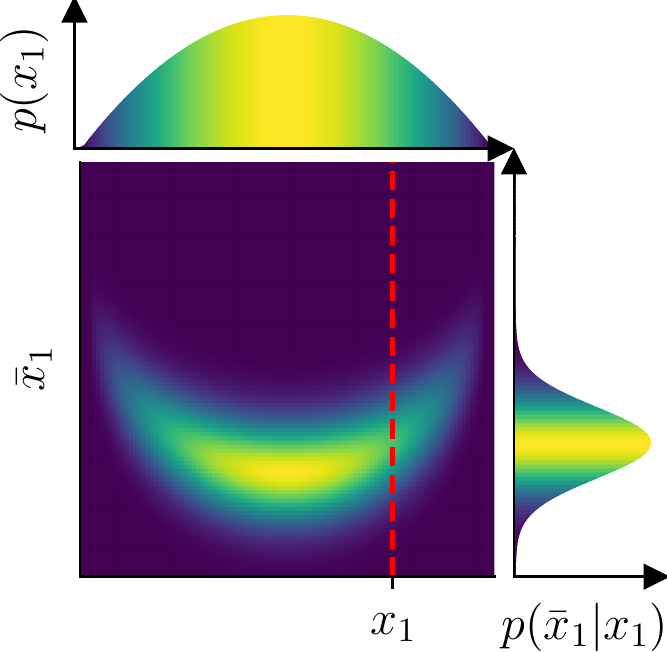}
    \hfill\null
    \null\hfill
    \includegraphics[width=0.475\linewidth]{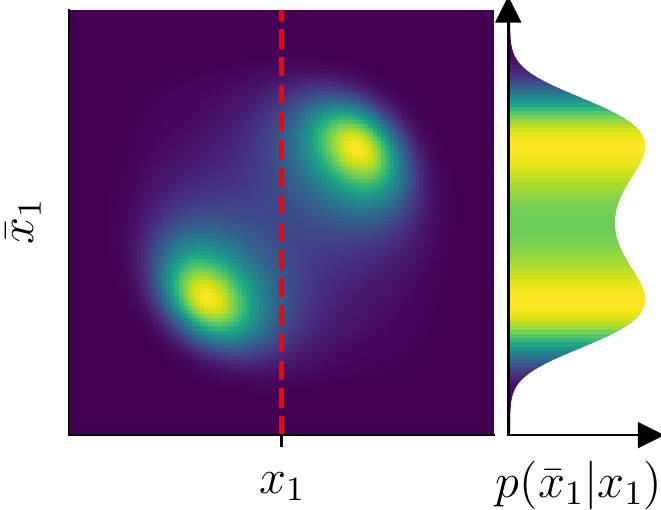}
    \hfill
    \includegraphics[width=0.475\linewidth]{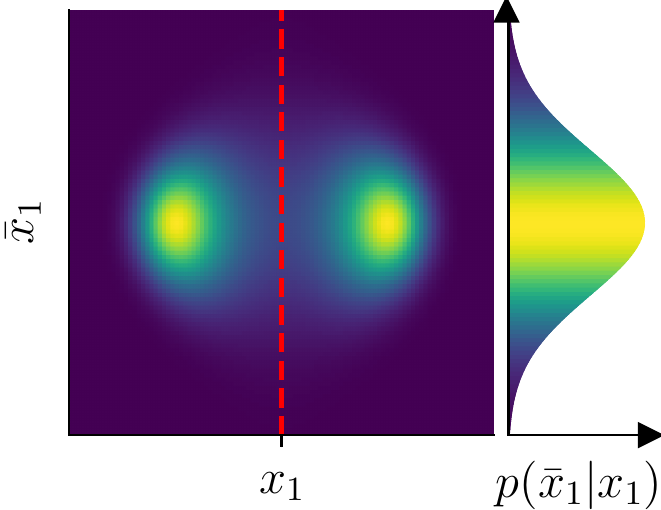}
    \hfill\null
    \caption{A globally log-concave distribution is conditionally log-concave (\emph{top left}), but the converse is not true (\emph{top right}): a non-convex support can have convex vertical slices (and horizontal projection). Conditional log-concavity also depends on the choice of orthogonal projectors: a distribution can fail to be conditionally log-concave in the canonical basis (\emph{bottom left}) but be conditionally log-concave after a rotation of $45$ degrees (\emph{bottom right}).}
    \label{fig:logconcave_examples}
\end{figure}

The following subsections provide bounds on the learning and sampling errors $\bar\epsilon^L_j$ and $\bar\epsilon^S_j$ for CSLC models. To simplify notations, in the following we drop the index $j$ and replace $p_{\bar \theta_j}(\bar x_j | x_j)$ with $p_{\bar \theta}(\bar x | x)$. We shall suppose that the dimension $d_J = \dim(x_J)$ is sufficiently small so that $x_J$ can be modeled and generated with any standard algorithm with small errors $\epsilon^L_J$ and $\epsilon^S_J$ ($d_J = 1$ in our numerical experiments).

\subsection{Learning Guarantees with Score Matching}
\label{sec:conditional_learning}

Fitting probabilistic models $p_{\bar\theta}(\bar{x}|x)$ by directly minimizing the $\mathrm{KL}$ errors $\bar{\epsilon}^L$ is computationally challenging because of intractable normalization constants. Strong log-concavity enables efficient yet accurate learning via a tight relaxation to score matching.

There exist several frameworks to fit a parametric probabilistic model to the data, most notably the maximum-likelihood estimator of a general energy-based model $\pii_{\bar \theta}(\bar{x}|x) = Z_{\bar\theta}^{-1}(x) e^{-\bar{E}_{\bar\theta}(x,\bar{x})}$, where $\bar{E}_{\bar\theta}$ is an arbitrary parametric class. 
This is computationally expensive due to the need to estimate the gradients of the normalization constants $-\nabla_{\bar\theta} \log Z_{\bar\theta} = \expectt[\pii_{\bar\theta}]{\nabla_{\bar\theta} \bar{E}_{\bar\theta}}$ during training, which requires the ability to sample from $p_{\bar\theta}(\bar x|x)$. 
An appealing alternative which has enjoyed recent popularity is \emph{score matching} \citep{hyvarinen2005estimation}, which instead minimizes the Fisher Divergence $\mathrm{FI}$:
\begin{align*}
    \ell(\bar\theta)
    &= \expect[x]{\frac12\mathrm{FI}_{\bar x}(\pii(\bar x|x) \,\|\, \pii_{\bar\theta}(\bar x|x))} \\ 
    &= \expect[x,\bar x]{\frac12\| {-\nabla_{\bar{x}} \log p(\bar{x}|x)} - \nabla_{\bar{x}} \bar{E}_{\bar\theta}(x,\bar{x})\|^2}.    
\end{align*}
With a change of variables we obtain
\begin{equation}
    \label{eq:lossfunctionscorematching}
    \ell(\bar\theta) = \expect[x,\bar x]{\frac12\|\nabla_{\bar{x}} \bar{E}_{\bar\theta} \|^2 - \Delta_{\bar{x}} \bar{E}_{\bar\theta}} + \mathrm{cst},
\end{equation}
showing that $\ell({\bar\theta})$ can be minimized from available samples without estimating normalizing constants or sampling from $p_{\bar\theta}$. 
Indeed, given i.i.d.\ samples $\{(\bar{x}^{1},x^{1}), \dots, (\bar{x}^{n},x^{n}) \}$ from $p(\bar{x},x)$, the empirical risk $\hat\ell(\bar \theta)$ associated with score matching on $p(\bar x|x)$ is given by
\begin{align}
\label{eq:SM_erm}
    \null\!\!\!\hat{\ell}({\bar\theta}) &=\frac1n \sum_{i=1}^n \paren{ \frac12 \|\nabla_{\bar x} \bar{E}_{\bar\theta}({x}^{i},\bar{x}^{i})\|^2 - \Delta_{\bar x} \bar{E}_{\bar\theta}({x}^{i}, \bar{x}^{i} ) } .
\end{align}

The score-matching objective avoids the computational barriers associated with normalization and sampling in high-dimensions, at the expense of defining a weaker metric than the KL divergence. This weakening of the metric is quantified by the log-Sobolev constant $\rho[p]$ associated with $\pii$. It is the largest $\rho>0$ such that $\mathrm{KL}(q\,\|\,\pii) \leq \frac{1}{2\rho}\mathrm{FI}(q\,\|\,\pii)$ for any $q$. Learning via score matching can therefore be seen as a relaxation of maximum-likelihood training, whose tightness is controlled by the log-Sobolev constant of the hypothesis class \citep{koehler2022statistical}. This constant can be exponentially small for general multimodal distributions, making this relaxation too weak. A crucial exception, however, is given  by SLC distributions (or small perturbations of them), as shown by the Bakry-Emery criterion \citep[Definition 1.16.1]{bakry2014analysis}: if $\alpha[p_{{{\bar\theta}}}(\bar x | x) ]\geq \bar{\alpha}>0$ for all $x$, or equivalently if $\nabla^2_{\bar x} \bar{E}_{\bar\theta} \succeq \bar{\alpha} \Id$ for all $x,\bar x$, then $\rho[p_{{{\bar\theta}}}(\bar x|x)] \geq \bar{\alpha}$ for all $x$, and therefore
\begin{equation}
\label{eq:bakry}
     \bar\epsilon^L \leq \frac{1}{\bar{\alpha}} \ell(\bar\theta). %\expect[x]{\mathrm{KL}(p(\bar x | x)\,\|\, p_{{{\bar\theta}}}(\bar x | x) )} \leq \frac{1}{2\alpha} \expect[x]{\mathrm{FI}(p(\bar x | x) \,\|\, p_{{{\bar\theta}}}(\bar x | x))}. 
\end{equation}
We remark that while \cref{eq:bakry} does not make explicit CSLC assumptions on the reference distribution $p$, a consistent learning model implies that the conditional distribution $p(\bar{x} | x)$ is arbitrarily well approximated (in $\mathrm{KL}$ divergence) with SLC distributions---thus justifying the structural CSLC assumption on the target.

\subsection{Score Matching with Exponential Families}
\label{sec:sm_exponential}

In numerical applications, one cannot minimize the true score-matching loss $\ell$ as only a finite amount of data is available. We now show that a similar control as \cref{eq:bakry} can be obtained for the empirical loss minimizer, whenever prior information enables us to define low-dimensional exponential models for $p_{\bar \theta}(\bar{x} | x)$ with good accuracy. It also provides a control on the critical parameter $\bar{\alpha}$, addressing the optimization and statistical errors.

%For ease of exposition, we now drop the $j$ subindex, and 
We consider a linear model $\bar{E}_{\bar\theta}(x,\bar{x}) = {\bar\theta}\trans \bar{\Phi}(x,\bar{x})$ with a fixed potential vector 
$\bar \Phi(x,\bar x) \in \mathbb{R}^{m}$ ($m$ is thus the number of parameters), and the corresponding minimization of the (conditional) score matching objective in \cref{eq:SM_erm}. Thanks to this linear parameterization, it becomes a convex quadratic form $\hat{\ell}({\bar\theta}) = \frac12 {\bar\theta}\trans \hat{H} {\bar\theta} - {\bar\theta}\trans \hat{g}$, with
\begin{align*}
    \hat{H} &= \frac1n\sum_{i=1}^n {\nabla_{\bar x} \bar\Phi(x^{i},\bar{x}^{i}) \nabla_{\bar x} \bar\Phi(x^{i},\bar{x}^{i}) \trans} \in \R^{m\times m}, \\
    \hat{g} &= \frac1n\sum_{i=1}^n {\Delta_{\bar x} \bar \Phi(x^{i},\bar{x}^{i})} \in \R^m.   
\end{align*}
It can be minimized in closed-form by inverting the Hessian matrix: ${\hat{\bar\theta}} = \hat{H}^{-1} \hat{g}$. As discussed, the sampling and learning guarantees of the model critically rely on the CSLC property, which is ensured as long as ${\hat{\bar\theta}} \in \Theta_{\bar{\alpha}} := \{ {\bar\theta} \,|\, \nabla_{\bar{x}}^2 \bar{E}_{\bar\theta}(x,\bar{x}) \succeq \bar{\alpha} \Id ,\ \forall\,(x,\bar{x}) \}$ with $\bar{\alpha}>0$.

The following theorem leverages the finite-dimensional linear structure of the score-matching problem to establish fast non-asymptotic rates of convergence, controlling the excess risk \emph{in $\mathrm{KL}$ divergence}. 
\begin{restatable}[Excess risk for CSLC exponential models]{theorem}{SMexpo}
\label{prop:learn2}
    Let ${\bar\theta}^\star=\arg\min \ell({\bar\theta})$ and $\hat{\bar\theta} = \arg\min \hat\ell(\bar\theta)$. Assume:
    \begin{enumerate}[label={(\roman*)}]
    \item $\bar\theta^\star \in \Theta_{\bar \alpha}$ for some $\bar \alpha>0$, 
    \item ${H}=\expect{\nabla_{\bar x} \bar\Phi \nabla_{\bar x} \bar\Phi\trans} \succeq \eta \Id$ with $\eta>0$,
    \item the sufficient statistics $\bar{\Phi}$ satisfy moment conditions \ref{ass:normbound_our}, regularity conditions \ref{ass:smallball}, and $\nabla \bar{\Phi}_k(x,\bar x)$ is $M_{\bar{\Phi}}$-Lipschitz for any $k\leq m$ and all $x$ (see \Cref{sec:proofs_sec3}).
    \end{enumerate}
    Then when $n > m$, the empirical risk minimizer $\hat{\bar\theta}$ satisfies
    \begin{equation}
    \label{eq:scprop_main}
        \hat{\bar\theta} \in \Theta_{\hat{\bar\alpha}}\ \text{with } \expect[(\bar x^{i}, x^{i})]{\hat{\bar\alpha}} \geq \bar\alpha - O\left(\eta^{-1} \sqrt{\frac{m}{n}}\right),
    \end{equation}
    and, for $t \ll \sqrt{m} \ell(\bar{\theta}^\star)$,
    \begin{equation}\label{eq:KLgen_main}
       \bar{\epsilon}^L \leq \frac{\ell({\bar\theta}^\star)}{\bar \alpha}(1+t)
    \end{equation}
    with probability greater than $1 - \exp\left\{ -O(n \log (tn/\sqrt{m}))\right\} $ over the draw of the training data. The constants in $O(\cdot)$ only depend on moment and regularity properties of $\bar{\Phi}$.
\end{restatable}
The theorem provides learning guarantees for the empirical risk minimizer $\hat{\bar\theta}$ (compare \Cref{eq:bakry,eq:KLgen_main}), and hinges on three key properties: the ability of the exponential family to approximate the true conditionals at each block (i) with small Fisher approximation error $\ell({\bar\theta}^\star)$, (ii) with a sufficiently large strong log-concavity parameter $\bar\alpha$, and (iii) with a well-conditioned kernel $H$. In numerical applications, the number of parameters $m$ should be small enough to control the learning error for finite number of samples $n$, and to be able to compute and invert the Hessian matrix $\hat H$. We will define in \Cref{sec:section4} low-dimensional models that can approximate a wide range of multiscale physical fields.

The proof uses concentration of the empirical covariance $\hat H$, and combines both upper and lower tail probability bounds \citep{mourtada2022exact, vershynin2012close} to bound the expectation, similarly as known results for least-squares \citep{mourtada2022exact, hsu2012random}. The statistical properties of score matching under exponential families have been studied in the infinite-dimensional setting by \citet{sriperumbudur2013density,sutherland2018efficient}, where kernel ridge estimators achieve non-parametric rates $n^{-s}$, $s<1$. Compared to these, as an intermediate result, we achieve the optimal rate in $\mathrm{FI}$ divergence in $n^{-1}$ directly with the ridgeless estimator (\Cref{eq:FIgen}). The key assumption is {\it (i)}, namely that the optimal model in the exponential family is SLC. Since our structural assumption on the target $p$ is precisely that its conditionals are SLC, it is reasonable to expect this to be generally true. For instance, this is the case if the model is well specified ($p = p_{\bar\theta^\star}$).

\subsection{Sampling Guarantees with MALA}
\label{sec:conditional_sampling}

We illustrate the efficient sampling properties of CSLC distributions by focusing on a reference sampler given by the Metropolis-Adjusted Langevin Algorithm (MALA) with algorithmic warm-start, which enjoys well-understood convergence properties in this case: 
\begin{restatable}[MALA Sampling, {{\citet[Theorem 5.1]{altschuler-chewi-2023-mala-algorithm-warm-start}}}]
{proposition}{malasampling}
    \label{prop:malascaling}
    Suppose that $\bar\alpha \Id \preceq \nabla^2_{\bar x} \bar E_{\bar\theta}(\bar{x}|x) \preceq \bar\beta \Id$ for all $\bar{x},x$, and let $\bar d = \mathrm{dim}(\bar{x})$. Then $N$ steps of MALA produce a sample $\bar{x}$ with conditional law $\hat{p}_{\bar \theta}(\bar{x}|x)$ satisfying
    \begin{equation*}
        \bar\epsilon^S \leq \exp\paren{-O\paren{\sqrt{\frac{N}{\sqrt{\bar d} \bar\beta/\bar\alpha}}}}.
    \end{equation*}
    %$\bar\epsilon^S \leq \epsilon$ for each $x$ provided that:
    % \begin{equation*}
    %     N = \widetilde{O}( (\beta/\alpha) \bar{d} \log^2(\epsilon^{-2})).
    % \end{equation*}
%        by running $N$ steps of MALA, where 
        %steps of MALA returns a sample $\hat{x}$ with conditional law $\hat{p}(\bar x | x)$ satisfying $$\mathrm{KL}(\hat{p}(\bar x| x) \,\|\, p(\bar x|x)) \leq \epsilon$$ for each $x$. 
        %{\color{red} What does $\widetilde O$ mean? Is the above formulation correct?}
\end{restatable}

MALA can thus be used to sample from CSLC distributions with an exponential convergence, whose mixing time $\widetilde O(\sqrt{\bar d} \bar\beta/\bar \alpha)$ is sublinear in the dimension $\bar d$ and linear in the condition number $\bar\beta/\bar\alpha$ of the Hessian $\nabla^2_{\bar x} \bar E_{\bar \theta}$.
We also note that similar guarantees will hold for other high-precision Metropolis-Hastings samplers, such as Hamilton Monte-Carlo. Together, \Cref{prop:global_error,prop:malascaling} and \Cref{prop:learn2} imply a control on the total accumulated error for CSLC exponential models.

%% file: section_wavpack.tex
\section{Wavelet Packet Conditional Log-Concavity}
\label{sec:section4}

The CSLC property depends on the choice of the projectors $(\bar{G}_j,G_j)$ which need to be adapted to the data. We show that for a class of stationary multiscale physical processes, CSLC models can be obtained with wavelet packet projectors. These models exploit the dominating quadratic interactions at high frequencies by splitting the frequency domain in sufficiently narrow bands. It reveals a powerful mathematical structure in this class of complex distributions.

\subsection{Energies with Scalar Potentials}

In the following, $x \in \R^d$ is a $\sqrt{d} \times \sqrt{d}$ image or two-dimensional field. We denote $x[i]$ the value of $x$ at pixel or location $i$. An important class of stationary probability distributions $p(x) = Z^{-1} e^{-E(x)}$ are defined in physics from an energy composed of 
a two-point interaction term $K$ plus a potential that is a sum of scalar potentials $v$:
\begin{equation}
\label{scal-en}
E(x) = \frac 1 2 x\trans K x + \sum_i v(x[i]) .
\end{equation}
The matrix $K$ is a positive symmetric convolution operator. \Cref{scal-en} generalizes both zero-mean Gaussian processes (if $v = 0$ then $K$ is the inverse covariance) and distributions with i.i.d.\ components (if $K=0$ then $v$ is the negative log-density of the pixel values). The energy Hessian is given by
\begin{equation}
    \label{eq:scalar_potential_hessian}
    \nabla^2_x E(x) = K + \diag\paren{v''(x[i])}_i.
\end{equation}
If $v''(t) < 0$ for some $t \in \R$ then we may get negative eigenvalues for some $x$, in which case the energy is not convex. 

\Cref{scal-en} provides models of a wide class of physical phenomena \citep{marchand_wavelet_2022}, including ferromagnetism. An important example is the $\varphi^4$ energy in physics, which is a non-convex energy allowing to study phase transitions and explain the nature of numerical instabilities \citep{10.1093/oso/9780198834625.001.0001}. It has a kinetic energy term defined by $K = -\beta \Delta$ where $\Delta$ is a discrete Laplacian that enforces spatial regularity, and its scalar potential is $v(t) = t^4 - (1+2\beta)t^2$. It has a double-well shape which pushes the values of each $x[i]$ towards $+1$ and $-1$, and is thus non-convex. $\beta$ is an inverse temperature parameter. In the thermodynamic limit $d \to \infty$ of infinite system size, the $\varphi^4$ energy has a phase transition at $\beta_c\approx 0.68$ \citep{doi:10.1142/S0129183116501084}. At small temperature ($\beta \geq \beta_c$), the local interactions in the energy give rise to long-range dependencies. Gibbs sampling then ``critically slows down'' \citep{article,Sethna2021StatisticalME} due to these long-range dependencies.

Fast sampling can nevertheless be obtained by exploiting conditional strong log-concavity. Assume that there exists $\gamma > 0$ such that $v''(t) \geq -\gamma$ for all $t \in \R$. It then follows that $\nabla^2_x E \succeq K - \gamma \Id$. We can thus obtain a convex energy by restricting $K$ over a subspace where its eigenvalues are larger than $\gamma$. The convolution $K$ is diagonalized by the Fourier transform, with positive eigenvalues that we write $\hat K(\omega)$ at all frequencies $\omega$. The value $\hat K(\omega)$ typically increases when the frequency modulus $|\omega|$ increases. A convex energy is then obtained with a projector over a space of high-frequency images, as shown in the following proposition.

\begin{proposition}[Conditional log-concavity of scalar potential energies]
    \label{th:conditionally-log-concave-scalar-potential}
    Consider the energy defined in \cref{scal-en} and assume that $-\gamma \leq v'' \leq \delta$ for some $\gamma, \delta > 0$ and that $\hat K(\omega) = \lambda |\omega|^\eta$ for some $\eta > 0$. Let $\bar G_1$ be an orthogonal projector over a space of signals whose Fourier transform have a support included over frequencies $\omega$ such that $|\omega| \geq |\omega_0|$ with $|\omega_0| > (\gamma/\lambda)^{1/\eta}$. Then the conditional probability $p(\bar x_1 | x_1)$ is strongly log-concave for all $x_1$.
\end{proposition}

The proof is in \Cref{sec:proof_th_cslc} and relies on a direct calculation of the Hessian of the conditional energy. This proposition proves that we obtain a strongly log-concave conditional distribution $p(\bar x_1 | x_1)$ with a sufficiently high-frequency filter $\bar G_1$. It is illustrated in the bottom row of \Cref{fig:logconcave_examples} on a simplified two-dimensional example inspired from the $\varphi^4$ energy. The distribution has two modes $x=(1,1)$ and $x=(-1,-1)$, and the Fourier coefficients are computed with a $45$ degrees rotation: $x_1 = (x[1] + x[2])/\sqrt{2}$ and $\bar x_1 = (x[2] - x[1])/\sqrt{2}$, which leads to a log-concave conditional distribution.

Multiscale physical fields with scalar potential energies (\ref{scal-en}) are often self-similar over scales, in the sense that lower-frequency fields $x_j$ can also be described with an energy in the form of \cref{scal-en}, with different parameters \citep{wilson1971renormalization}. This explains why \Cref{th:conditionally-log-concave-scalar-potential} can be iterated to obtain a CSLC decomposition. 
For $\varphi^4$ energies, the range of $\bar G_1$ is non-empty as soon as $\beta \geq \frac12$, which includes the critical temperature $\beta_c \approx 0.68$ (though $\delta=\infty$). At the critical temperature, $x_1$ is further described by the same parameters $K$ and $v$ as $x$, so that a complete CSLC decomposition is obtained by iteratively selecting projectors $\bar G_j$ which isolate the highest frequencies of $x_{j-1}$.

\Cref{th:conditionally-log-concave-scalar-potential} can be extended to general energies 
\begin{equation*}
    E(x) = \frac12 x\trans K x + V(x),    
\end{equation*}
by assuming that the Hessian $\nabla^2 V(x)$ is bounded above and below. Conditional log-concavity may then be found by exploiting dominating quadratic energy terms with a PCA of $K$. We believe that this general principle may hold beyond the case of scalar potential energies (\ref{scal-en}) considered here.

\subsection{Wavelet Packets and Renormalization Group}

We now define wavelet packet projectors $G_j$ and $\bar G_j$, which are orthogonal projectors on localized zones of the Fourier plane. They are computed by convolutions with conjugate mirror filters and subsamplings \cite{coifman1992wavelet}, described in \Cref{sec:wavelet_packets}. These filters perform a recursive split of the frequency plane illustrated in \Cref{plot:WP2dPhi4}.

The wavelet packet $\bar{G}_j$ is a projector on a high-frequency domain, whereas $G_j$ is a projection on the remaining lower-frequency domain. An orthogonal wavelet transform is a particular example, which decomposes the Fourier plane into annuli of about one octave bandwidth, as shown in the top left and bottom panels of \Cref{plot:WP2dPhi4}. However, it may not be sufficiently well localized in the Fourier domain to obtain strictly convex energies. The frequency localization is improved by refining this split, as illustrated on the top right panel of \Cref{plot:WP2dPhi4}. Each $\bar G_j$ then performs a projection over a frequency annulus whose bandwidth is a half octave. Wavelet packets can adjust the frequency bandwidth to $2^{-{M+1}}$ octave for any integer $M \geq 1$. It allows reducing the support of $\bar G_j$, which is necessary to obtain a CSLC decomposition according to \Cref{th:conditionally-log-concave-scalar-potential}.

\begin{figure}[t]
\centering
\includegraphics[width=0.40\textwidth]{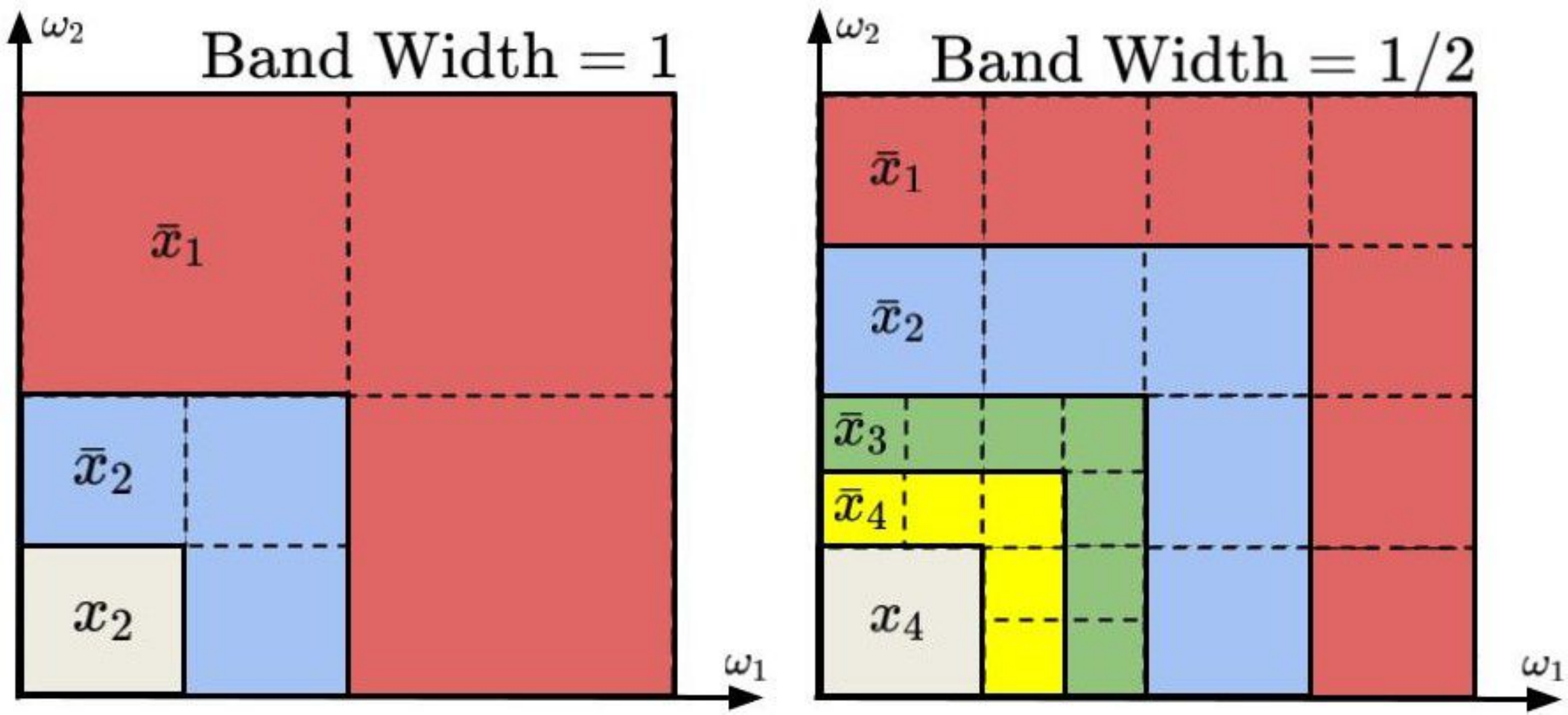}
\includegraphics[width=0.40\textwidth]{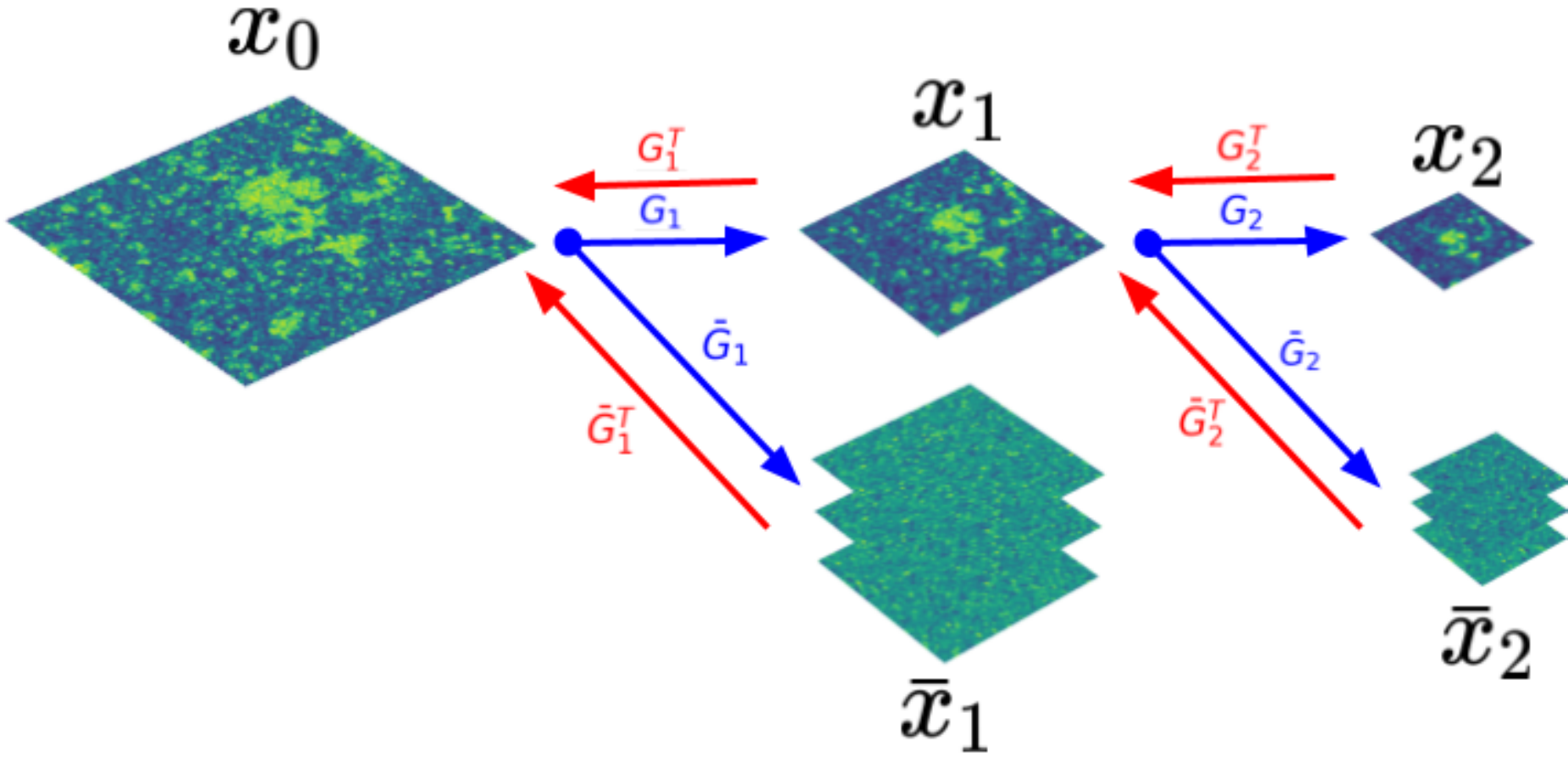}
 \caption{\emph{Top:} frequency localization of the decomposition $(x_J, \bar{x}_J, \dots, \bar{x}_1)$ with wavelet packet projectors of $1$ (\emph{left}) and $1/2$ (\emph{right}) octave bandwidths. \emph{Bottom:} iterative decomposition of $x=x_0$ with $(\bar{G}_j,G_j)$ implementing a wavelet packet transformation over $J=2$ layers of $1$ octave bandwidth.}
\label{plot:WP2dPhi4}
\end{figure}

\subsection{Multiscale Scalar Potentials}

The probability distribution $p(x)$ is approximated by $p_\theta(x) = p_{\theta_J}(x_J) \prod_{j=1}^J p_{\bar \theta_j}(\bar x_j|x_j)$, where each $x_j$ and $\bar x_j$ are computed with wavelet packet projectors $G_j$ and $\bar G_j$. We introduce a parameterization of $p_{\bar \theta_j}$ with scalar potential energies, following \citet{marchand_wavelet_2022}. We shall suppose that the dimension $d_J = \mathrm{dim}(x_J)$ is sufficiently small so that $p(x_J)$ may be approximated with any standard algorithm ($d_J = 1$ in our numerical experiments).

The self-similarity property of multiscale fields with scalar energies motivates the definition of each $p_{\bar \theta_j}(\bar x_j | x_j)$ with an interaction energy
\begin{align}
\nonumber
\bar E_{\bar \theta_j}(x_j,\bar x_j) &= \frac 1 2 \bar x_j\trans \bar K_{j} \bar x_j + \bar x_j \trans \bar K'_{j} x_j + \sum_i \bar v_j (x_{j-1}[i]) \\
&= \bar \theta_j^T \bar \Phi_j(x_j,\bar x_j),
\label{barphijeq}
\end{align}
which derives from the fact that $p(x_{j-1})$ defines an energy of the form (\ref{scal-en}) \citep{marchand_wavelet_2022}. $\bar \Phi_j$ captures the interaction terms and performs a parametrized approximation of $\bar v_j$, defined in \Cref{app:phibar}. 
%calculation in \Cref{sec:proof_stability_phi4} shows that under appropriate assumptions, this parameterization is stable, in the sense that the condition number of the score matching Hessian is bounded uniformly in $j$. It implies that its inversion is numerically stable and provides control on the learning errors $\bar\epsilon_j^L$ through \Cref{prop:learn2}.

% \begin{theorem}[Stability of score matching for scalar potential energies]
% \label{th:scalar_short}
%     Let $H_j = \nabla^2_{\bar\theta_j} \ell_j =\expect{\nabla_{\bar x_j} \bar\Phi_j \nabla_{\bar x_j} \bar\Phi_j\trans}$.
%     If the power spectrum of $x$ satisfies a power law decay, and under some regularity conditions on the wavelet packet $(G_j, \bar G_j)$ and $p$ detailed in \Cref{sec:proof_stability_phi4}, there exists $C\geq 1$ such that:
%     \begin{equation*}
%         \forall j,\ \kappa(H_j)\leq C.
%     \end{equation*}
% \end{theorem}

The parameters $\bar \theta_j$ are estimated from samples by inverting the empirical score matching Hessian as in \Cref{sec:sm_exponential}. We generate samples from the resulting distribution $p_\theta$ by sampling from $p_{\theta_J}$ and then iteratively from each $p_{\bar \theta_j}$ with MALA. The learning and sampling algorithms are summarized in \Cref{sec:sm_mala_algorithms}. Additionally, \Cref{sec:energy_estimation} explains that a parameterized model of the global energy (\ref{scal-en}), which is crucial for scientific applications, can be recovered with free-energy score matching.

%% file: section_numerics.tex
\section{Numerical Results}
\label{sec:numerics}

This section demonstrates that a wavelet packet decomposition of $\varphi^4$ scalar fields and weak-lensing cosmological fields defines strongly log-concave conditional distributions. It allows efficient learning and sampling algorithms, and leads to higher-resolution generations than in previous works.

\subsection{$\varphi^4$ Scalar Potential Energy}

We learn a wavelet packet model of $\varphi^4$ scalar fields at different temperatures, using the decomposition and models presented in \Cref{sec:section4}. The wavelet packet exploits the conditionally strongly log-concave property of $\varphi^4$ scalar fields (\Cref{th:conditionally-log-concave-scalar-potential}) to obtain a small error in the generated samples, as shown in \Cref{sec:section3}. We first verify qualitatively and quantitatively that this error is small.

We evaluate the wavelet packet model at three different temperatures, which have different statistical properties: $\beta = 0.50$, the ``disorganized'' state, $\beta = 0.68 \approx \beta_c$ the critical point, and $\beta = 0.76$ the ``organized'' state. The computational efficiency of our approach enables generating high-resolution $128 \times 128$ images, as opposed to $32 \times 32$ in \citet{marchand_wavelet_2022}. Indeed, learning the model parameters for $64\times64$ images with score matching takes seconds on GPU, whereas doing the same with maximum likelihood takes hours on CPU (as sequential MCMC steps are not easily parallelized). The generated samples are shown in \Cref{plot:Phi4Synth} and are qualitatively indistinguishable from the training data. The experimental setting is detailed in \Cref{sec:experimental_details}.

A distribution $p(x)$ having a scalar potential energy (\ref{scal-en}) is a maximum-entropy distribution constrained by second-order moments and hence by the power spectrum, and by the marginal distribution of all $x[i]$. These statistics specify the matrix $K$ and the scalar potential $v(t).$ Our model $p_\theta$ also has a scalar potential energy in this case. To guarantee that $p_\theta = p$, it is thus sufficient to show that they have the same power spectrum and same marginal distributions. We perform a quantitative validation of generated samples by comparing their marginal densities and Fourier spectrum with the training data. \Cref{plot:Phi4Synth} shows that these statistics are well recovered by our model.

\begin{figure}[t]
\centering
\includegraphics[width=0.48\textwidth]{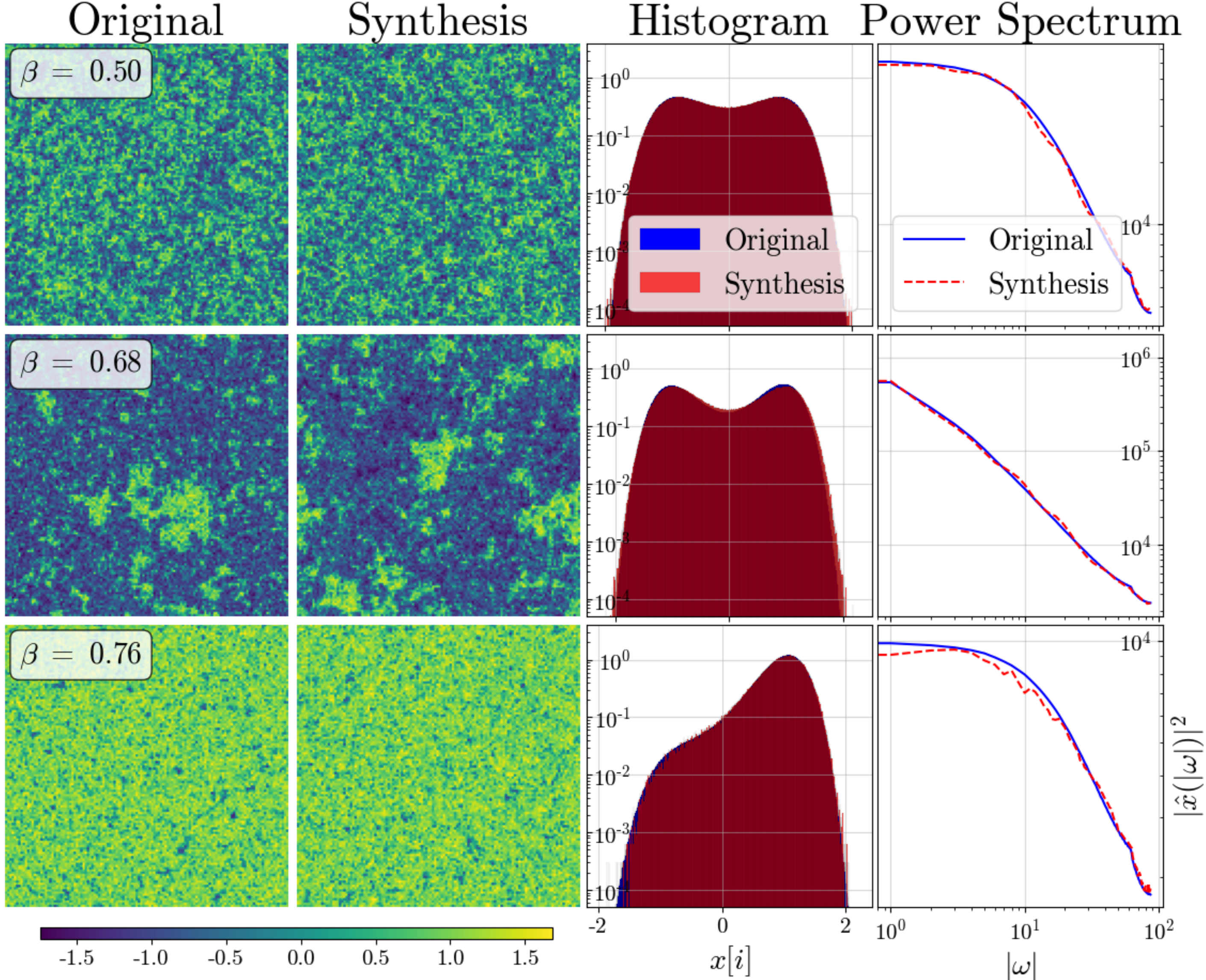}
\vspace{-0.5cm}
\caption{ Comparison between training and generated samples for $\varphi^4$ energies. \emph{In columns:} training samples, generated samples, histograms of marginal distributions $p(x[i])$ and power spectrum. \emph{In rows:} disorganized state $\beta = 0.50$, critical point $\beta = 0.68 \approx \beta_c$, and organized state $\beta = 0.76$.}
\label{plot:Phi4Synth}
\end{figure}

\subsection{Conditional Log-Concavity}

We numerically verify that $\varphi^4$ at critical temperature is CSLC (\Cref{def:marginal_logconcave}), with appropriate wavelet packet projectors. It amounts to verifying that the eigenvalues of the conditional Hessian $\nabla^2_{\bar{x_j}}\bar{E}_{\bar{\theta}_j}(x_j, \bar x_j)$ are positive for all $x_j$ and $\bar x_j$. We can restrict $x_j$ to typical samples from $p(x_j)$. However, it is important that the Hessian be positive even for $\bar x_j$ outside of the support of $p(\bar x_j|x_j)$. Indeed, negative eigenvalues occur at local directional maxima of the energy, rather than minima which would correspond to most likely samples. We thus evaluate the Hessian at $\bar{x}_j = 0$, which is expected to be such an adversarial point.

\Cref{plot:ConvexEigs} shows distributions of eigenvalues of $\nabla^2_{\bar x_j} \bar E_{\bar\theta_j}$ for decompositions $(\bar G_j, G_j)$ of various frequency bandwidths. It shows that the smallest eigenvalues become larger and eventually cross zero as the frequency bandwidth of $\bar G_j$ becomes narrower, as predicted by \Cref{th:conditionally-log-concave-scalar-potential}. Furthermore, the condition number of the Hessian becomes smaller as eigenvalues concentrate towards their mean. %We obtain similar results for all $j$ (not shown) due to the scale-invariance of the $\varphi^4$ field at the critical temperature.

As shown in \cref{eq:scalar_potential_hessian}, both the quadratic part $K$ and the scalar potential $v$ contribute to the Hessian. As a way to visualize both contributions, we define the equivalent scalar potential $v^0$ as
%\begin{equation*}
    % \frac12 x\trans K x + \sum_n v(x(n)) = \frac12 x\trans K_0 x + \sum_n v_0(x(n))
    $v^0(t) = v(t) + \frac{\Tr(K)}{2d} t^2$.
%\end{equation*}
It corresponds to extracting the mean quadratic value $\Tr(K)/2d\, \norm{x}^2$ from the quadratic part and reinterpreting it as a scalar potential. This allows visualizing the average energy on a pixel value when neglecting spatial correlations. The right panel of \Cref{plot:ConvexEigs} compares these equivalent scalar potentials for the energy $E_j$ of $x_j$ and the conditional energy $\bar E_j$. It shows that the non-convex double-well potential in the global energy becomes convex after the conditioning. It verifies \Cref{th:conditionally-log-concave-scalar-potential}, as the mean quadratic value becomes larger when we restrict $K$ to a subspace of high-frequency signals. %This is expanded upon in \Cref{sec:energy_estimation}, which explains how the non-convex $\varphi^4$ energy is represented with convex energies at each scale.
%convj 2-2.pdf

\begin{figure}[t]
\centering
\includegraphics[width=0.45\textwidth]{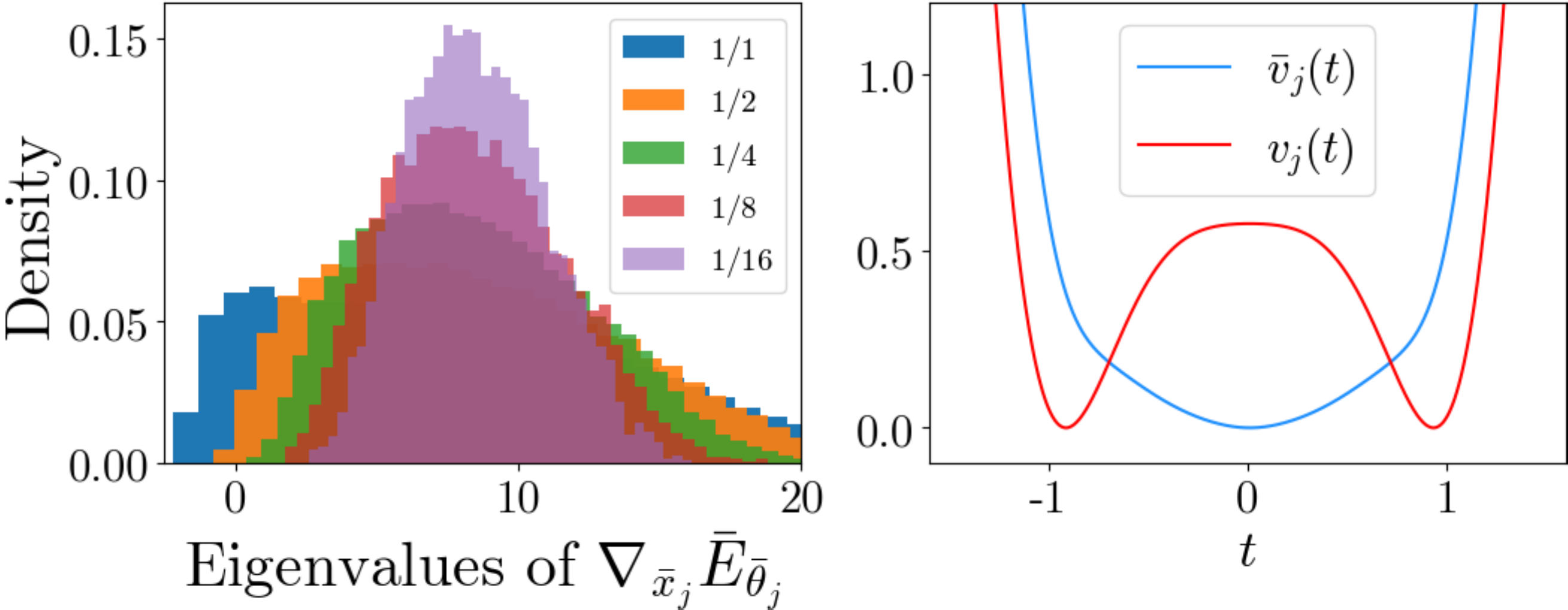}
\caption{ Conditional strong log-concavity of $\varphi^4$ at critical temperature. All scales $j$ yield similar results. \emph{Left:} distribution of eigenvalues of $\nabla^2_{\bar x_j} \bar{E}_{\bar{\theta}_j}$ for different frequency bandwidths ($j=1$ is shown). \emph{Right:} equivalent scalar potentials $v_j$ and $\bar{v}_j$ ($j=3$ is shown).}
\label{plot:ConvexEigs}
\end{figure}

We also verify the sampling efficiency predicted by \Cref{prop:malascaling}. As we cannot evaluate the $\mathrm{KL}$ divergences $\bar \epsilon_j^S$, we rather compute the decorrelation mixing time $\bar{\tau}$, a measure of the number of steps of conditional MALA to reach a given fixed error threshold averaged over all scales $j$. The precise definition is given in \Cref{sec:mala_mixing_time}. We compare it with the decorrelation mixing time $\tau$ of MALA on the non-convex global energy $E$.

Sampling maps of size $\sqrt d \times \sqrt d$ from the global $\varphi^4$ energy $E$ at the critical temperature requires a number of steps $\tau \sim d^{1.0}$ \citep{10.1093/oso/9780198834625.001.0001}. This phenomena is known as critical slowing down \citep{article,Sethna2021StatisticalME}, a consequence of long-range correlations. We numerically show that our algorithm does not suffer from it. \Cref{plot:Mix} indeed demonstrates an empirical scaling $\bar\tau \sim d^{0.35}$. Note that this is not directly comparable with \Cref{prop:malascaling} as the decorrelation mixing time defines a different convergence rate than the $\mathrm{KL}$ mixing time.

\begin{figure}[t]
\centering
\includegraphics[width=0.40\textwidth]{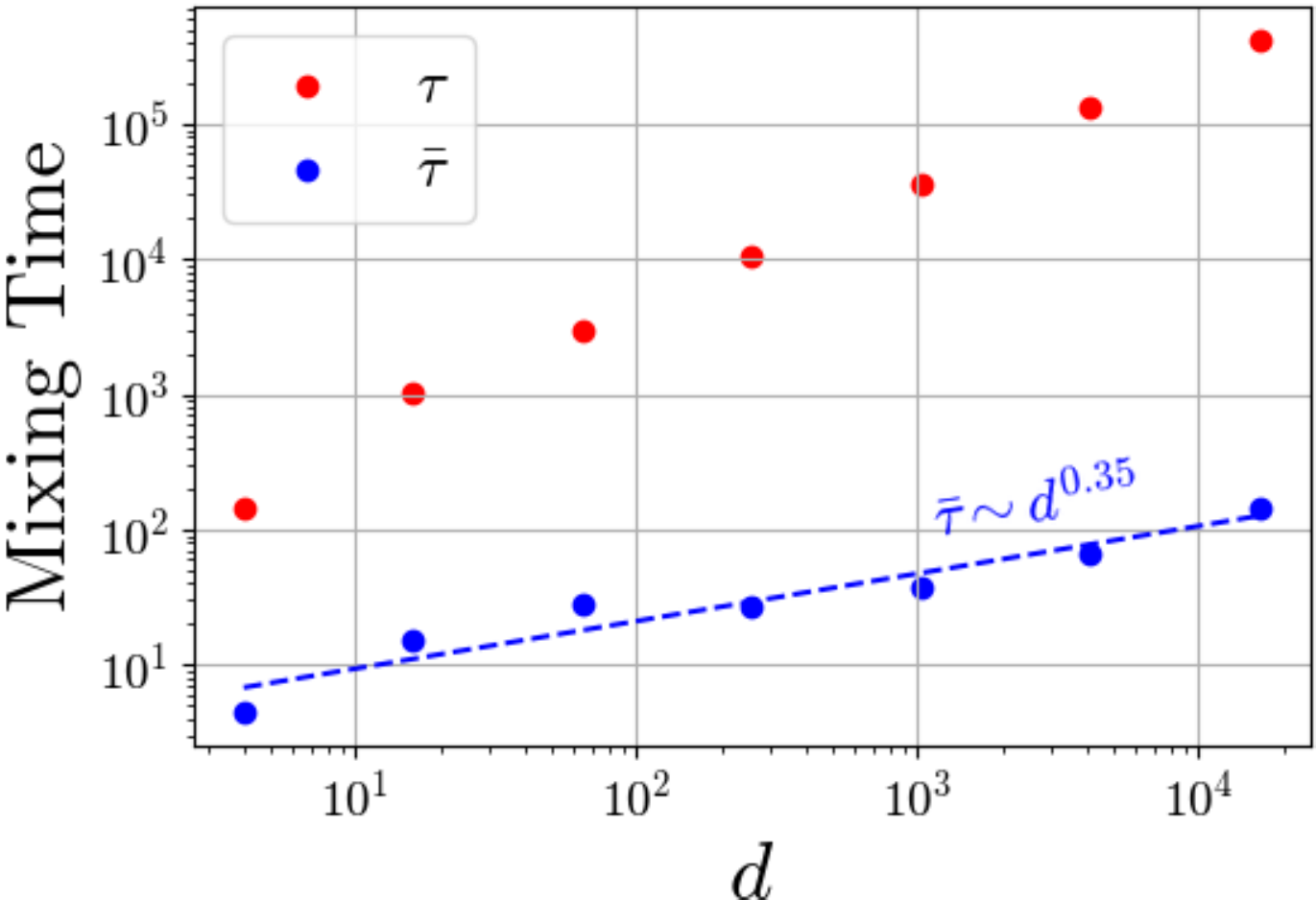}
\caption{Mixing times for direct ($\tau$) and conditional ($\bar\tau$) sampling for $\varphi^4$ at critical temperature.}
\label{plot:Mix}
\end{figure}

\subsection{Application to Cosmological Data}

We now apply our algorithm to generate high-resolution weak lensing convergence maps \citep{Bartelmann_2001,Kilbinger_2015} with an explicit probability model. Weak lensing convergence maps measure the bending of light near large gravitational masses on two-dimensional slices of the universe. We used simulated convergence maps computed by the Columbia lensing group \citep{PhysRevD.94.083506,PhysRevD.97.103515} as training data. They simulate the next generation outer-space telescope {\it Euclid} of the European Space Agency \citep{laureijs2011euclid}, which will be launched in 2023 to accurately determine the large scale geometry of the universe governed by dark matter. Estimating the probability distribution of such maps is therefore an outstanding problem \citep{marchand_wavelet_2022}. We demonstrate that the CSLC property is surprisingly verified in this real-world example, and can be used to efficiently model and generate these complex fields.

We use the same models and algorithms as for the $\varphi^4$ energy. The experimental setting is detailed in \Cref{sec:experimental_details}.
\Cref{plot:WeakSynth} shows that our generated samples are visually  highly similar to the training data. Quantitatively, they have nearly the same power spectrum. The marginal distribution of all $x[i]$ are also nearly the same, with a long tail corresponding to high amplitude peaks, which are typically difficult to reproduce. As opposed to microcanonical simulations with moment-matching algorithms \citep{10.1093/mnras/stab2102}, we compute an explicit probability distribution model, which is exponential. As a maximum-entropy model, it has a higher entropy than the true distribution, and therefore does not suffer from lack of diversity. By relying on the CSLC property, we can use the fast score-matching algorithm and compute $128\times 128$ images, at four times the $32\times 32$ resolution than with a maximum-likelihood algorithm used in \citet{marchand_wavelet_2022}.

\Cref{plot:wpot} shows the equivalent scalar potentials of the conditional energies at all scales, which are all convex and thus verify the CSLC property of weak lensing model. It demonstrates that this property can be used to efficiently model and generate high-resolution complex data.

\begin{figure}[t]
\centering
\includegraphics[width=0.45\textwidth]{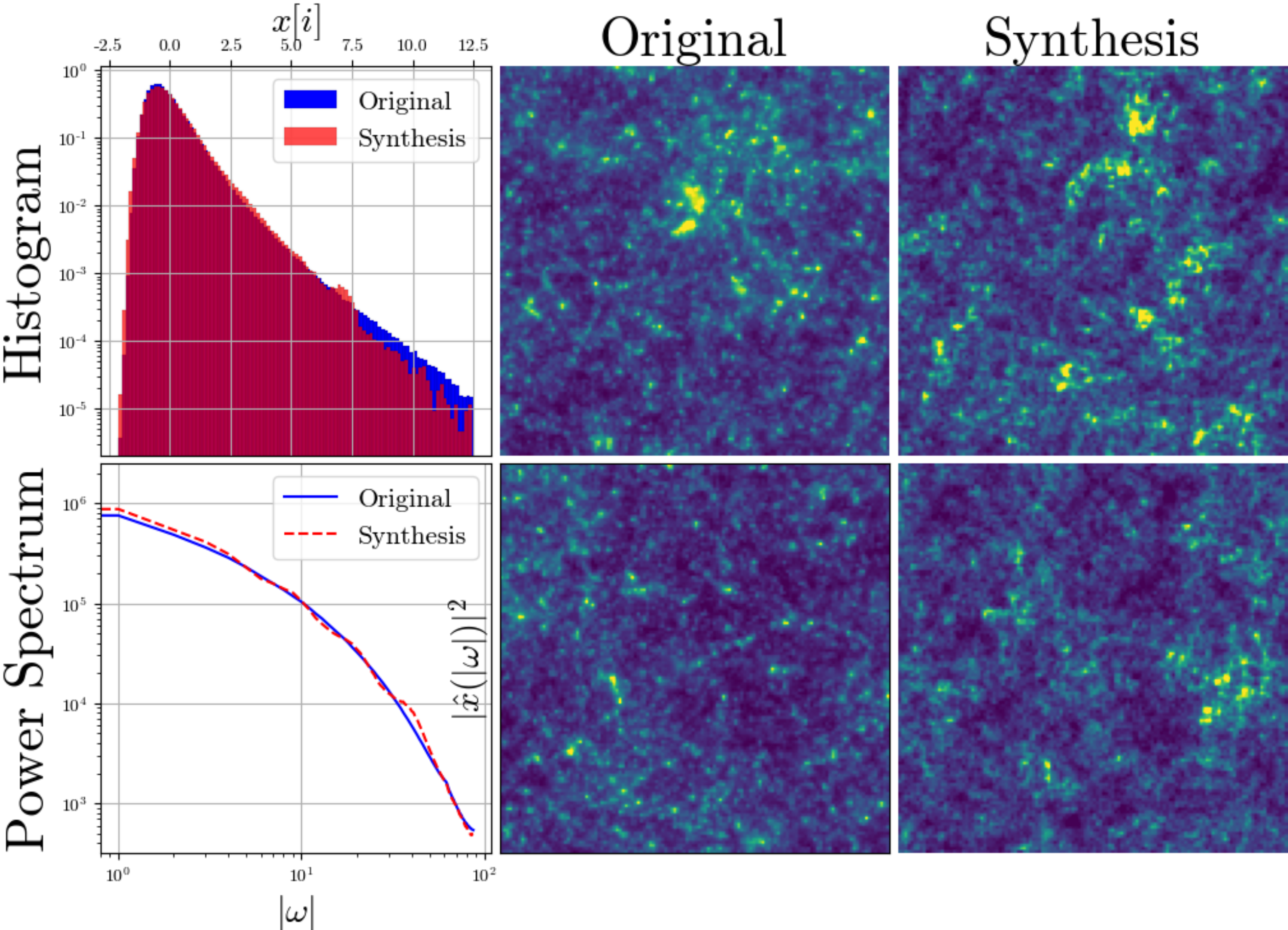}
\caption{Comparison between training and generated samples for weak-lensing maps. \emph{Upper left:} histograms of marginal distributions $p(x[i])$. \emph{Lower left:} power spectrum. \emph{Center:} training samples. \emph{Right:} generated samples.}
\label{plot:WeakSynth}
\end{figure}

\begin{figure}[t]
\centering
\includegraphics[width=0.40 \textwidth]{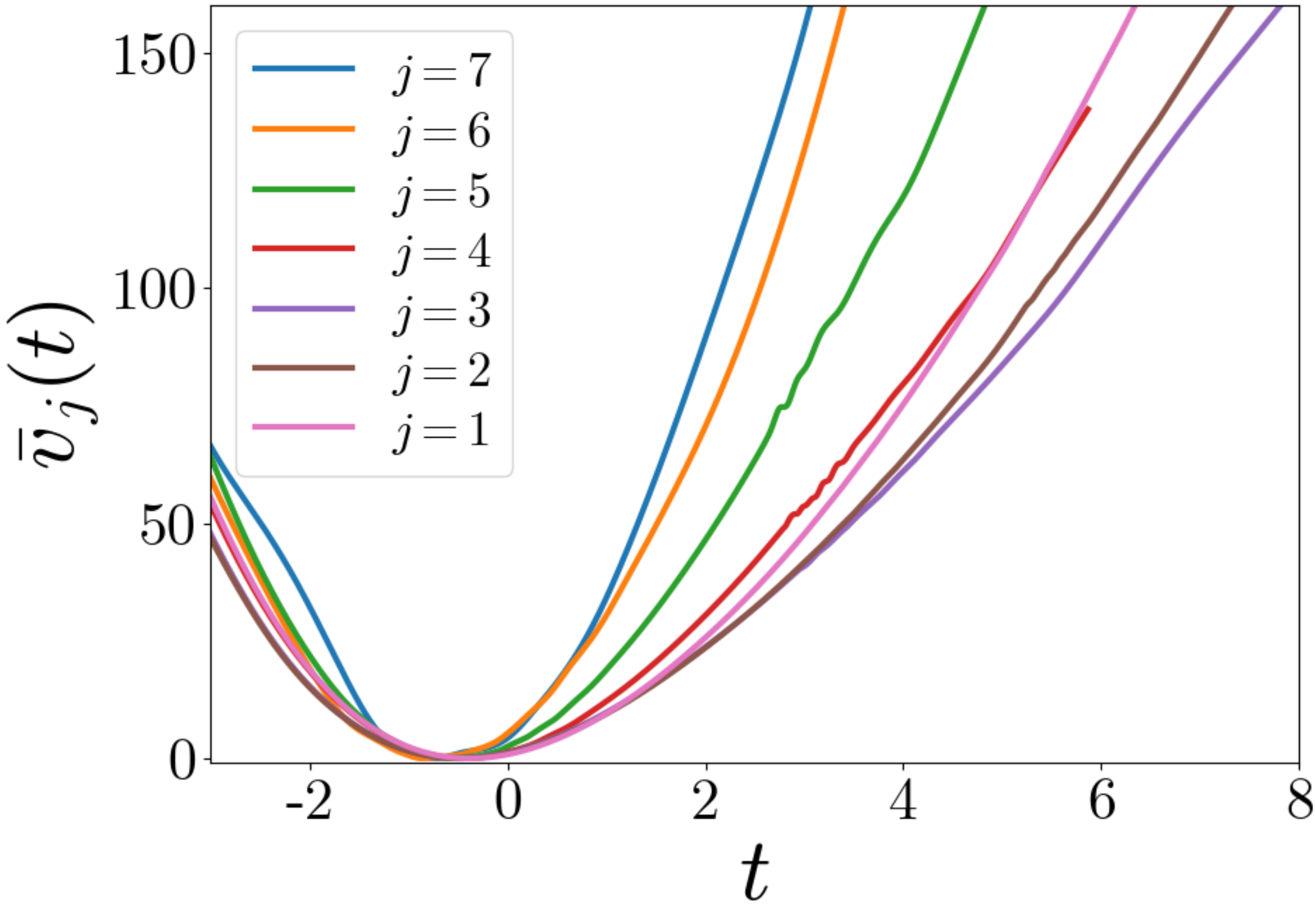}
\caption{Equivalent scalar potentials $\bar{v}_j$ at each scale $j$ for weak-lensing maps (normalized for viewing purposes).}
\label{plot:wpot}
\end{figure}

%% file: section_conclusion.tex
\section{Discussion}

We introduced conditionally strongly log-concave (CSLC) models and proved that they lead to efficient learning with score matching and sampling with MALA, while controlling errors. These models rely on iterated orthogonal projections of the data that are adapted to its distribution. We showed mathematically and numerically that complex multiscale physical fields satisfy the CSLC property with wavelet packet projectors. The argument is general and relies on the presence of a quadratic (kinetic) energy term which ensures strong log-concavity at high-frequencies. It provides high-quality and efficient generation of high-resolution fields even when the underlying distribution is unknown. The CSLC property guarantees diverse generations without memorization issues, which is critical in scientific applications.

CSLC models can be extended by introducing latent variables. The guarantees of \Cref{sec:section3} extend to the case where the data is a marginal of a CSLC distribution. A notable example is a score-based diffusion model, for which the data $x = x_0$ is a marginal of a higher-dimensional process $(x_t)_t$ whose conditionals $p(x_{t-\delta}|x_t)$ are approximately Gaussian white when $\delta$ is small, thus introducing a tradeoff between the number of terms in the CSLC decomposition and the condition number of its factors.  
Score diffusion is a generic transformation, but it assumes that the score $\nabla_{x_t} \log p(x_t)$ can be estimated with deep networks at any $t \geq 0$ \citep{song2020score,ho2020denoising}. For high-resolution images, the score estimation often uses conditional multiscale decompositions with or without wavelet transforms \citep{saharia2021image,ho2022cascaded,nichol2021beatgans,guth_wavelet_2022}. Understanding the log-concavity properties of natural image distributions under such transformations is a promising research avenue to understand the effectiveness of score-based diffusion models.

% A particularly powerful framework to overcome the limitations of Gibbs representations 
% is given by conditional models: if a target distribution $\pi(x)$ is a marginal of a measure of the form 
% $p_\theta(z) = p_0(z_0) \prod_{j=1}^J p_{\theta_j}(z_j | z_{0:j-1})$ for $z=(z_0,z_1, \ldots, z_J) \in \bar{\Omega}:=\Omega_1 \times \dots \times \Omega_J$, with each conditional $p_{\theta_j}(\cdot | z_{1:j-1})$ efficiently sampleable,  
% then sampling any marginal of $p_\theta(z)$ can be accomplished with complexity proportional to $J$. In particular, it is sufficient to consider $p_{\theta_j}(\cdot | z_{1:j-1})$ a strongly log-concave Gibbs distribution for each $j$. 
% This framework encompasses normalizing flows (where conditional models are deterministic transports of the form $p_{\theta_j}(z_j|z_{j-1}) = \delta(z_j-\Phi_j(z_{j-1}))$), {\color{red} Dirac ?} auto-regressive models (where the conditional models $p_{\theta_j}(z_j|z_{1:j-1})$ are low-dimensional, often $z_j$ one-dimensional), or score-based diffusion models (where the conditionals are Gaussian with means determined by certain time-dependent score functions, and $J$ is polynomial in the ambient dimension). 

%% file: appendix.tex
\appendix
%\section{Appendix}
\onecolumn

% \addcontentsline{toc}{section}{Appendix} % Add the appendix text to the document TOC
% \part{Appendix} % Start the appendix part
% \parttoc % Insert the appendix TOC

\appendixpage
% \addappheadtotoc
\startcontents[sections]
\printcontents[sections]{l}{1}{\setcounter{tocdepth}{1}}

% \Cref{sec:wavelet_packets} gives details on the definition of the wavelet packet projectors $(G_j, \bar G_j)$ used in numerical experiments. \Cref{sec:sm_mala_algorithms} presents the full score matching and MALA algorithms

% \begin{itemize}
%     \item Wavelet packet definition
%     \item Score Matching and MALA algorithms
%     \item Datasets and experimental details
%     \item Additional results
%     \item Energy estimation with free-energy modeling
%     \item Proofs of section 2
%     \item Proofs for scalar potential models (log-concavity + stability)
% \end{itemize}

\section{Definition of Wavelet Packet Projectors}
\label{sec:wavelet_packets}

The fast wavelet transform \cite{Mallat1989ATF} splits a signal in frequency into two orthogonal coarser signals, using two orthogonal conjugate mirror filters $g$ and $\bar{g}$. 

We review the construction of such filters in \cref{subsec:filters}. A description of the fast wavelet transform is then given in \cref{subsec:fastwavelettransform}. Finally, we define in \cref{subsec:waveletpacket} the wavelet packet \cite{coifman1992wavelet} projectors $({G}_j,\bar G_j)$  used in the numerical section \ref{sec:section4}. 

% We may reconstruct with the transpose filters the original field, from its decomposition.

% \begin{equation}
%     \left
%         \begin{array}{ll}
%             x_1 = gx_0 \\
%             \bar{x}_1 = \bar{g}x_0 \\
%             x_0 = \bar{g}\trans\bar{x_1}+gx_1
%         \end{array}
%     \right.
% \end{equation}

\subsection{Conjugate Mirror Filters}
\label{subsec:filters}
Conjugate mirror filters  $g$ and $\bar{g}$ satisfy the following orthogonal and reconstruction conditions:
\begin{equation}
    \label{eqn:ortho}
        \begin{array}{ll}
            g\trans \bar{g} = \bar{g}\trans g = 0 ,\\
            g\trans g +\bar{g}\trans\bar{g} = \Id .
        \end{array}
\end{equation}
In one dimension, the conditions (\ref{eqn:ortho}) are satisfied \cite{Mallat1989ATF} by discrete filters $(g(n))_{n\in\mathbb{Z}},(\bar{g}(n))_{n\in\mathbb{Z}}$ whose Fourier transforms
$\hat{g}(\omega)=\sum_n g(n)e^{-in\omega}$ and
$\hat{\bar g}(\omega)=\sum_n g(n)e^{-in\omega}$
satisfy
\begin{equation}
    \label{eqn:filters}
        \begin{array}{ll}
            \vert \hat{g}(\omega)\vert^2 +\vert \hat{g}(\omega+\pi)\vert^2 = 2 ,\\
            \hat{g}(0) = \sqrt{2} ,\\
            \hat{\bar{g}}(\omega) = e^{-i\omega}\hat{g}(\omega+\pi) .
            
        \end{array}
\end{equation}
We first design a low-frequency filter $g$ such that
$\hat g(\omega)$ satisfies (\ref{eqn:filters}), and then compute $\bar{g}$ with
\begin{equation}
    \label{eqn:filters_hi}
        \begin{array}{ll}
            \bar{g}(n) = (-1)^{1-n}g(1-n) .
        \end{array}
\end{equation}

The choice of a particular low pass filter $g$ is a trade-off between a good localization in space and a good localization in the Fourier frequency domain. Choosing a perfect low-pass filter $g(\omega) = \mathds{1}_{\omega\in[-\pi/2,\pi/2]}$ leads to Shannon wavelets, which are well localized in the frequency domain but have a slow decay in space. On the opposite, a Haar wavelet filter  $g(n)=\sqrt{2} \, \mathds{1}_{n\in\{0,1\}}$ has a small support in space but is poorly localized in frequency. 
Daubechies filters \cite{doi:10.1137/1.9781611970104} provide a
good joint localization both in the spatial and Fourier domains. The Daubechies-$4$ wavelet is shown in \Cref{plot:Db4Fourier}.

\begin{figure}[t]
\begin{center}

\includegraphics[width=0.45\textwidth]{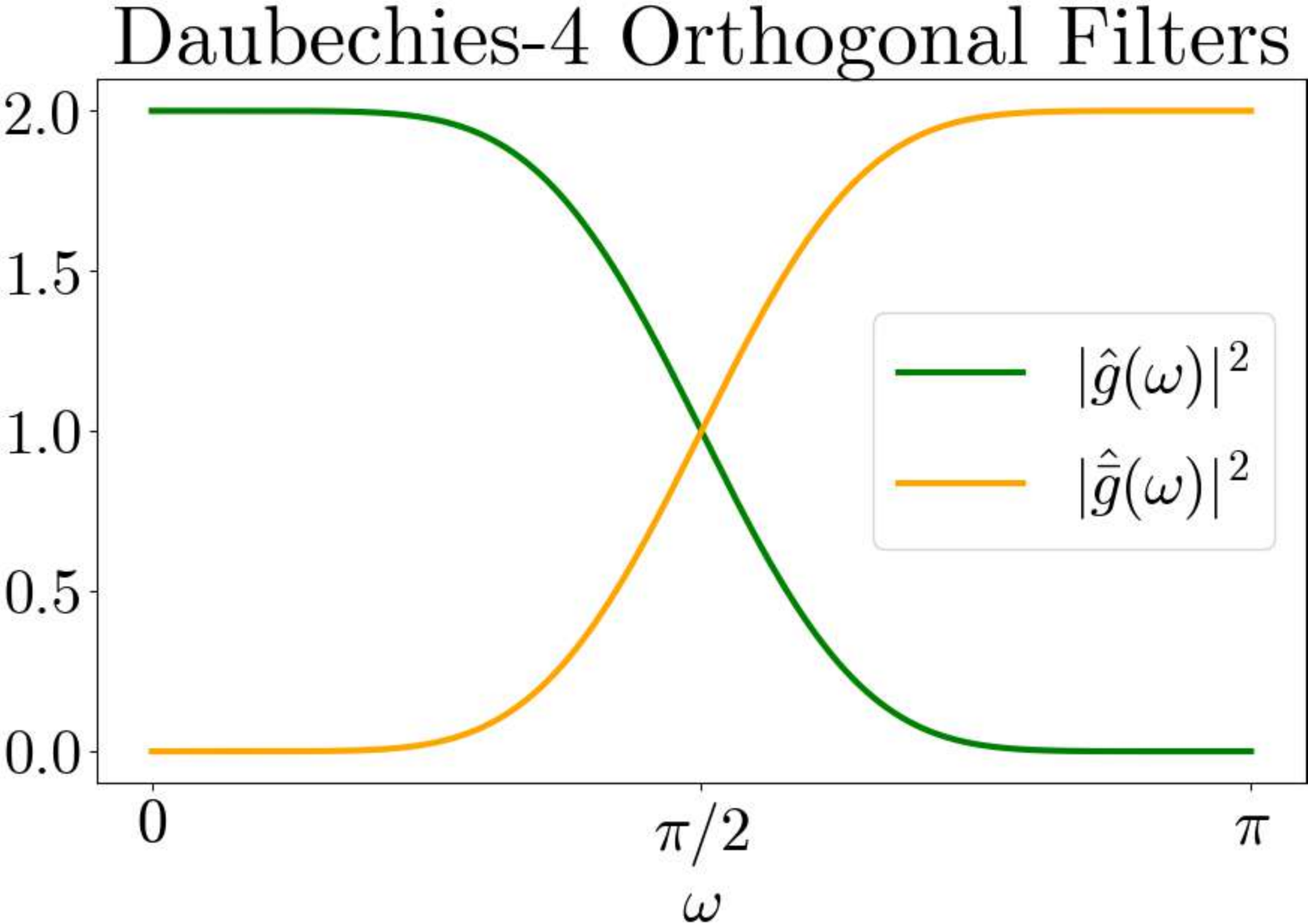}
\caption{Fourier transform of Daubechies-4 orthogonal filters $\hat g(\omega)$ (in green) and $\hat {\bar g}(\omega)$ (in orange).}
\label{plot:Db4Fourier}
\centering
\end{center}
\end{figure}

% \begin{minipage}{\linewidth}
%   \centering
%   \begin{minipage}{0.45\linewidth}
%       \begin{figure}[t]
%           \includegraphics[width=\linewidth]{Plots/WaveletTransform/Db4 big-1.pdf}
%           \caption{Fourier transform of Daubechies-4 orthogonal filters $\hat g(\omega)$ (in green) and $\hat {\bar g}(\omega)$ (in orange).}
%           \label{plot:Db4Fourier}
%       \end{figure}
%   \end{minipage}
%   \hspace{0.05\linewidth}
%   \begin{minipage}{0.45\linewidth}
%       \begin{figure}[t]
%           \includegraphics[width=\linewidth]{Plots/WaveletTransform/WP1D.pdf}
%           \caption{In one dimension a wavelet packet transform is obtain by cascading filtersing and subsamplings with the filters $g$ and $\bar g$ along a binary splitting tree which outputs $x_J$ and the different $\bar x_j$ for $j \geq J$.}
%           \label{plot:WP1d}
%       \end{figure}
%   \end{minipage}
% \end{minipage}

In two dimensions (for images), wavelet filters which satisfy the orthogonality conditions in (\ref{eqn:ortho}) can be defined as separable
products of the one-dimensional filters $g$ and $\bar{g}$ \cite{Mallat}, applied on each coordinate. It defines one low-pass filter $g_2$ and $3$ high-pass filters $\bar{g}_2 = (\bar{g}_2^k)_{1 \leq k\leq 3}$:
\begin{equation}
        \begin{array}{c}
            g_2(n_1, n_2) = g(n_1)g(n_2), \\
            \bar g_2^1(n_1,n_2) = g(n_1)\bar{g}(n_2) ,\\
            \bar g_2^2(n_1,n_2) = \bar{g}(n_1)g(n_2) ,\\
            \bar g_2^3(n_1,n_2) = \bar{g}(n_1)\bar{g}(n_2).
        \end{array}
\end{equation}
For simplicity we shall write $g$ and $\bar{g}$  the filters $g_2$ and $\bar{g}_2$. $\bar{g}$ outputs the concatenation of the 3 filters $\bar{g}_2^k$. 
% If $x = (x_1,...,x_d)$, $g= (g_1.....g_d)$ and $\bar{g}=(\bar{g}_1,...,\bar{g}_d)$

% \begin{equation}
%     \label{eqn:filters_hi}
%     \left
%         \begin{array}{ll}
%             gx(i_1,...i_d) = \sum_{n_1,....n_d} x(i_1,....i_d)\prod_{k}g_k(n_k-i_k)
%         \end{array}
%     \right.
% \end{equation}

\subsection{Orthogonal Frequency Decomposition}
\label{subsec:fastwavelettransform}
We introduce the
orthogonal decomposition of a signal $x_{j-1}$ with the low pass filter $g$ and the high pass filter $\bar{g}$, followed by a sub-sampling.
It outputs $(x_{j},\bar{x}_{j})$, which has the same dimension as $x_{j-1}$, defined in one dimension by
\begin{equation}
   \label{eqn:fastwav}
        \begin{array}{ll}
            x_{j} [p] =  \sum\limits_ {n\in\mathbb{R}^2} g[n-2p]x_{j-1}[n] ,\\
            \bar{x}_{j} [p] =  \sum\limits_{n\in\mathbb{R}^2} \bar{g}[n-2p]x_{j-1}[n] .
        \end{array}
\end{equation}
The inverse transformation is
\begin{equation}
   \label{eqn:fastwavinv}
        \begin{array}{l}
            x_j[p] = \sum\limits_{n\in\mathbb{R}^2} g[p-2n]x_{j+1}[n]+\sum\limits_{n\in\mathbb{R}^2} \bar{g}[p-2n]\bar{x}_{j+1}[n]   .
        \end{array}
\end{equation}

The orthogonal frequency decomposition in two dimensions is defined similarly. It decomposes a signal $x$ of size $\sqrt{d}\times\sqrt{d}$ into a low frequency signal and $3$ high frequency signals, each of size $ \frac{\sqrt{d}}{2} \times\frac{\sqrt{d}}{2}$. 

% The low frequency signal is then re-decomposed into a low frequency and a high frequency of size $2^{J-2}$. This leads to a decomposition $(\bar{x}_J,...,\bar{x}_1)$ where each $x_j$ is a signal of size $2^{J-j}$, representing the original signal at frequencies around $\omega_j =  2\pi2^{-j}$.

% In d-dimension, the field of size $(2^J)^d$ is decomposed into a field of low frequency of size $(2^{J-1})^d$, and $2^d-1$ fields of high frequencies of size $(2^{J-1})^d$.  

\subsection{Wavelet Packet Projectors}
\label{subsec:waveletpacket}
An orthogonal frequency decomposition projects a signal into high and low frequency domains. In order to refine the decomposition (by separating different frequency bands), wavelet packets projectors are obtained by cascading this orthogonal frequency decomposition.

The usual fast wavelet transform starts from a signal $\bar{x}_0$ of dimension $d$, decomposes it into a low-frequency $x_1$ and a high frequency $\bar{x}_1$, and then iterates this decomposition on the low-frequency $x_1$ only. It iteratively decomposes $x_{j-1}$ into the lower frequencies $x_j$ and the high-frequencies $\bar x_j$. The resulting orthogonal wavelet coefficients are $(\bar{x}_j,x_J)_{1\leq j\leq J}$. The resulting decomposition remains of dimension $d$. 

To obtain a finer frequency decomposition, we use the $M$-band wavelet transform \cite{Mallat}, a particular case of wavelet packets \cite{coifman1992wavelet}. It first applies the fast wavelet transform to the signal, and obtains $(\bar{x}_j,x_J)_{1\leq j\leq J}$. Each high-frequency output $\bar{x}_j$ undergoes an orthogonal decomposition using $g$ and $\bar g$. Then both outputs of the decomposition are again decomposed, and so on, $(M-1)$-times. The coefficients are then sorted according to their frequency support, and also labeled as $(\bar{x}_j,x_J)_{1\leq j\leq J'}$, with $J'=J2^{M-1}$, also referred to as $J$ in the main text.

The wavelet packet decomposition corresponds to first decomposing the frequency domain dyadically into octaves, and then each dyadic frequency band is further decomposed into $2^{M-1}$ frequency annuli. We say this decomposition corresponds to a $1/2^{M-1}$ octave bandwidth. Precisely, if $j=j'2^{M-1}+r$, then $\bar{x}_j$ has a frequency support over an annulus in the frequency domain, with frequencies with modulus of order  $2^{-j'}\pi(1- {2^{-M+1}}(r-1/2))$.
A two-dimensional visualization of the frequency domain can be found in \Cref{plot:WP2dPhi4}, for $M=1$ and $M=2$, corresponding to $1$ and $1/2$ octave bandwidths.

Figure \ref{plot:WP1d} shows the iterative use of $g$ and $\bar g$ used to obtain the decomposition, in one dimension, for $M = 2$. 
Note that the filters $\bar{g}$ and $g$  successively play the role of low-{} and high-pass filters because of the subsampling \cite{Mallat}.

% the high frequency output $\bar x$

% the wavelet packet transform
% \cite{coifman1992wavelet} also re-decomposes the wavelet
% high frequency outputs, M-times, with the orthogonal frequency
% decomposition using the filters $g$ and $\bar g$. 

% Each time 
% we subdecompose both the low frequency and high frequency output.
% We then order theses outputs images by frequency, into a decomposition that we write $(\bar{x}_j,x_{J'})_{1\leq j\leq J='J2^{Q-1} }$. If $j=j'2^{Q-1}+r$, then $\bar{x}_j$ has a frequency support over a frequency annulus where
% $|\omega|$ is of the order of $\pi2^{-j'}[1- {2^{-Q+1}}(r-1/2)]$.

We now introduce the corresponding orthogonal projectors
$G_j$ and $\bar{G}_j$, defined such that
\begin{equation}
   \label{eqn:fastpacket}
        \begin{array}{ll}
            \bar{x}_{j} = \bar{G}_j x_{j-1} ,\\

            x_{j} = G_j x_{j-1},
                 
        \end{array}
\end{equation}
where the $(\bar x_j )_j$, sorted in frequency, have been obtained trough the $M$-band wavelet transform, as described above, and $x_j$ refers to the signal reconstructed using $(x_J,\bar x_{j'})_{j'\geq j+1}$. 
Let us emphasize that the image $x_{j-1}$ is reconstructed from $x_j$ and the higher frequencies $\bar{x}_{j}$, and defined on a spatial grid which is either the same as $x_j$ or twice larger. For $M = 2$, \Cref{plot:x_j} shows that $x_0$ and $x_1$ are defined on the same grid, although $x_1$ has a lower-frequency support. Similarly $x_2$ and $x_3$ are both represented on the same grid, which is twice smaller, and so on.

The orthogonal projectors satisfy $G_j\trans G_j + \bar G_j\trans \bar G_j = \Id$.
We then have the following inverse formula:
\begin{equation}
   \label{eqn:fastpacket_inv}
        \begin{array}{l}
            x_{j-1}= G_j\trans x_j + \bar{G}_j\trans\bar{x_j}.
        \end{array}
\end{equation}
This decomposition using $G_j$ and $\bar{G}_j$ recursively splits the signal in frequencies, from high to low frequencies. 

% In 2d, the transformation works the same, and we decompose the original image into smaller maps sorted in frequency. We show \ref{plot:WP2d} how we can display theses smaller maps, sorted by the frequency they represent. We then select, like in the figure, the $\bar{x}_j$.   

\begin{figure}[t]

\begin{center}
\includegraphics[width=0.45\textwidth]{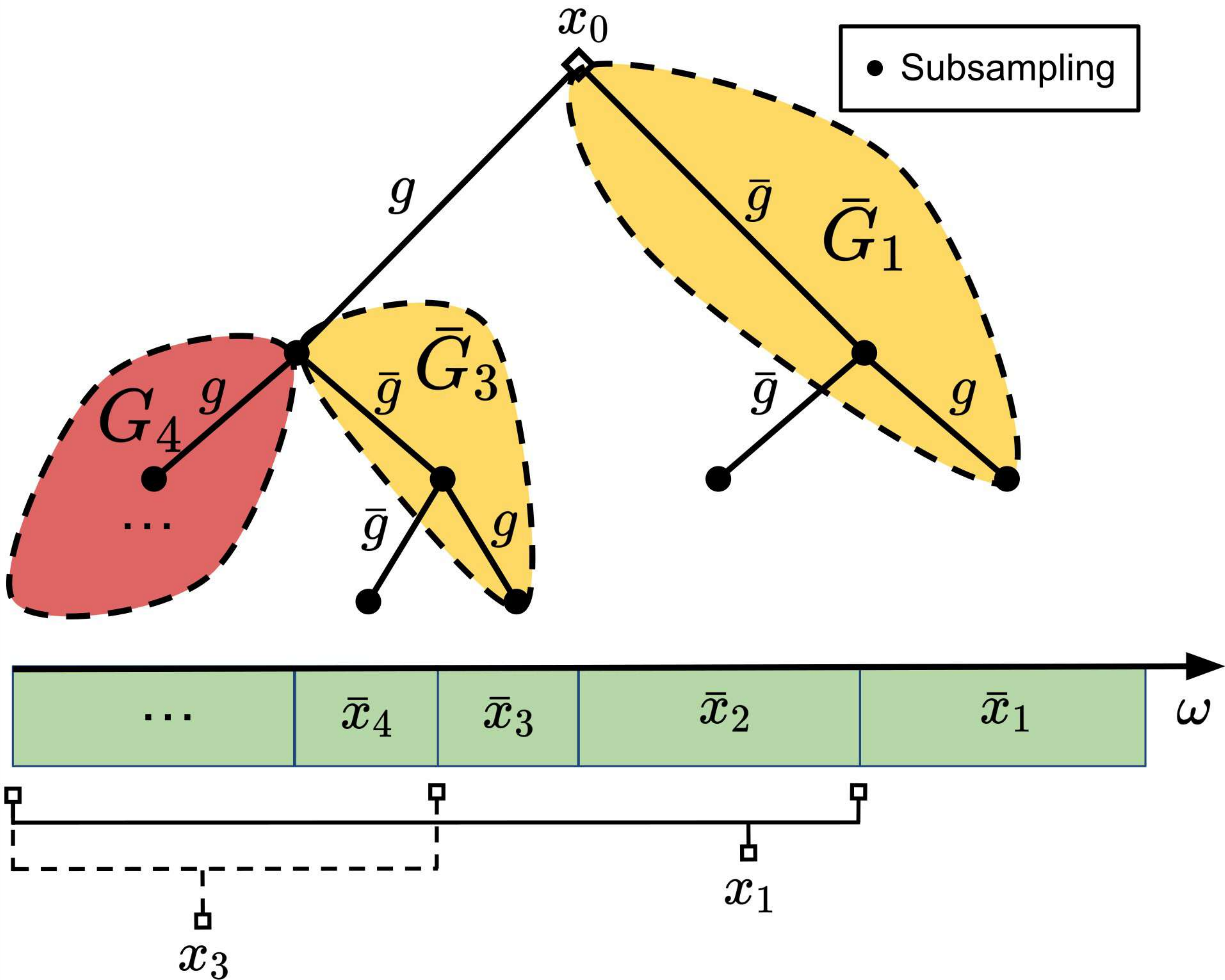}
\caption{In one dimension, a wavelet packet transform is obtain by cascading filterings and subsamplings with the filters $g$ and $\bar g$ along a binary splitting tree which outputs $x_J$ and $\bar x_j$ for $j \geq J$.}
\label{plot:WP1d}
\centering
\end{center}
\end{figure}

\begin{figure}[t]
\begin{center}
\includegraphics[width=0.45\textwidth]{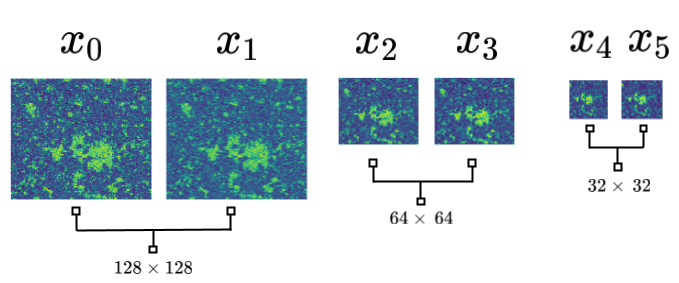}
\caption{Low-frequency maps $x_j$ for $M=2$ for a $\varphi^4$ realization.}
\label{plot:x_j}
\centering
\end{center}
\end{figure}

\section{Score Matching and MALA Algorithms for CSLC Exponential Families}
\subsection{Multiscale Energies}
\label{app:phibar}

This section introduces the explicit parametrization of the energies $\bar E_{\bar \theta_j}$ and $E_{\theta_J}$.

The conditional energies $\bar{E}_{\bar{\theta}_j}(x_j,\bar x_j)$
are defined with a bilinear term which represents the interaction between $x_j$ and $\bar x_j$ and a scalar potential:
 \begin{equation}
    \label{eqn:Ejdix}
         \bar{E}_{\bar{\theta}_j}(x_j,\bar{x}_j) = \frac{1}{2}\bar{x}_j\trans\bar{K}_j\bar{x}_j+\sum\limits_{l> j} \bar{x}_j\trans\bar{K}'_{l,j}\bar{x}_{j+l} + \sum_i \bar{v}_j(x_{j-1}[i]),
 \end{equation}
 with $x_{j-1} = \bar G_j\trans \bar x_j + G_j\trans x_j.$
\Cref{eqn:Ejdix} is an equivalent reparametrization of \cref{barphijeq}. Considering $(\bar x_l)_{l>j}$ instead of $x_j$ allows fixing some coefficients of the $\bar K'_{l,j}$ to zero instead of learning them. First, we set $\bar K'_{l,j} = 0$ if $\bar x_j$ and $\bar x_{j+l}$ are not defined on the same spatial grid. In the sequel, sums over $l$ only refer to theses terms, which differ depending on the wavelet decomposition. We enforce spatial stationarity by averaging the bilinear interaction terms across space. We further kept only the non-negligible terms which correspond to neighboring frequencies and neighboring spatial locations.
As displayed in \Cref{plot:sub}, $\bar{x}_j$ is composed of sub-bands $\bar{x}_j^k$. We kept the interaction terms $\bar{x}_j^k[i] \bar{x}_{j+l}^{k+\delta k}[i+\delta i]$ for $l \in \{0, 1\}$, $\delta k \in \{0,1\}$, and $\delta i \in \{0,1,2,3,4\}^2$, which correspond to local interactions in both space and frequency.

\begin{figure}[t]
\begin{center}   
\includegraphics[width=0.45\textwidth]{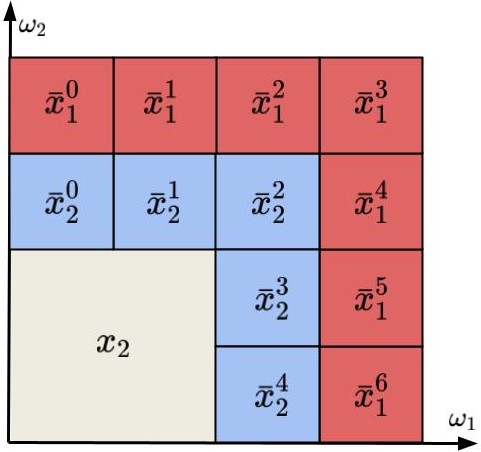}
\caption{Sub-bands of $\bar x_j$ for a wavelet packet decomposition with a half-octave bandwidth.}
\label{plot:sub}
\end{center}
\end{figure}
 
The scalar potential $\bar v_j(t)$ is decomposed on a family of predefined functions $\rho_{k,j}(t)$:
 \begin{equation}
    \label{eqn:vjdix}
         \bar{v}_j(t) = \sum\limits_k\bar{\alpha}_{k,j} \, {\rho}_{k,j}(t).
 \end{equation}
% The interaction parameter vector is defined by
% $\bar{\theta}_j = (\bar{K}_j,\bar{K}'_{l,j},\bar{\alpha}_{k,j})_{k,j}$ and the interaction potential vector is
%  \begin{equation}
%     \label{eqn:Potjdix}
%          \bar{\Phi}_j(x_j,\bar{x_j}) = \Big(\frac{1}{2}\bar{x}_j \bar{x}_j\trans,\bar{x}_j \bar{x}_{j+l}\trans, \rho_{k,j} (x_{j-1}) \Big)_{k,l}.
%  \end{equation}
${\rho}_{j,k}$ is defined in order to expand the scalar potential $\bar v_j$ which captures the marginal distributions of the $x_{j-1}[i]$, which do not depend on $i$ due to stationarity. We divide this marginal into $N$ quantiles. Each ${\rho}_{k,j}$ is chosen to be a regular bump function having a finite support on the $k$-th quantile. This parametrization performs a pre-conditioning of the score matching Hessian.

Let $\rho$ be a bump function 
with a support in $[-{1}/{2},{1}/{2}]$. For each $j$, let
$a_{j,k}$ and $l_{j,k}$ be respectively the center and width of the $k$-th quantile of the marginal distribution of $\bar x_j$, we define
\begin{equation}
   \label{eqn:rho_k}
        {\rho}_{k,j}(t)=  \,l_k\sqrt{N} \, \rho\paren{\frac{t -a_{j,k}}{l_{j,k}}},
\end{equation}
with the condition
\begin{equation}
   \label{eqn:amplitude_factor}
        \norm{\rho'}_2^2 = \frac{1}{  \norm{\bar{G}_j}_2^2 },
\end{equation}
in order to balance the magnitude of the scalar potentials with the quadratic potentials. %More details about this normalisation are to be found in \cref{?}.

The potential vector is thus
\begin{equation}
\label{eqn:potentialdef}
\bar \Phi_j(x_j, \bar x_j) = \paren{ \sum\limits_{i} \bar x_j^k[i] \bar x_{j+l}^{k + \delta k}[i + \delta_i], \sum\limits_{i}\rho_{k',j}(x_{j-1}[i])}_{0 \leq l \leq 1, 0 \leq \delta k \leq 1, 0 \leq \delta_i \leq 4,1 \leq k' \leq N}.
\end{equation}

Similarly, we define $E_{\theta_J}$ as the sum of a quadratic energy and a scalar potential:
\begin{equation}
\label{eqn:EJdix}
     {E}_{{\theta}_J}(x_J) = \frac{1}{2}x_J\trans K_Jx_J+ \sum_i v_J (x_J [i]).
\end{equation}
The bilinear interaction terms are averaged across space to enforce stationarity.
The scalar potential $v_J(t)$ is also decomposed over a family of predefined functions $\rho_{k,J}(t)$:
\begin{equation}
\label{eqn:vJdix}
     v_J(t) = \sum\limits_k \alpha_{k,J} \, \rho_{k,J}(t),
\end{equation}
defined similarly as above.
This yields a potential vector
 \begin{equation}
    \label{eqn:PotJdix}
         \Phi_J(x_J) = \paren{\sum_i x_J[i] x_J[i + \delta_i],\rho_{k,J} (x_J)}_{0\leq\delta i \leq 4, 1 \leq k \leq N},
 \end{equation}
leading to
 \begin{equation}
    {E}_{{\theta}_J}(x_J) = \theta_J\trans \Phi_J (x_J),
 \end{equation}
with $\theta_J =(K_J,\alpha_{k,J})_k$.

\subsection{Pseudocode}
\label{sec:sm_mala_algorithms}

The procedure to  learn the parameters $(\bar \theta_j)_j$ of the conditional energies $\bar E_{\bar \theta_j}(x_j, \bar x_j)$ by score matching is detailed in \Cref{alg:SM}. The procedure to generate samples from the distribution $p_\theta(x)$ with MALA is detailed in \Cref{alg:sampling}.

\begin{algorithm}[H]
\caption{Score matching for exponential families with CSLC distributions}\label{alg:SM}
\begin{algorithmic}
    \REQUIRE Training samples $(x^{i})_{1 \leq i \leq n}$.
    \STATE Initialize $x^{i}_0 = x^{i}$ for $1 \leq i \leq n$.
    \FOR{$j = 1$ \TO $J$}
        \STATE Decompose $x_j^{i} \leftarrow G_j x_{j-1}^{i}$ and $\bar x_j^{i} \leftarrow \bar G_j x_{j-1}^{i}$ for $1 \leq i \leq n$.
        \STATE Compute the score matching quadratic term ${H}_j \leftarrow \frac1n\sum_{i=1}^n {\nabla_{\bar x_j} \bar\Phi_j(x^{i}_j,\bar{x}^{i}_j) \nabla_{\bar x_j} \bar\Phi_j(x^{i}_j,\bar{x}^{i}_j) \trans} \in \R^{m\times m}$.
        \STATE Compute the score matching linear term ${g}_j \leftarrow \frac1n\sum_{i=1}^n {\Delta_{\bar x_j} \bar \Phi_j(x^{i}_j,\bar{x}^{i}_j)} \in \R^m$.
        \STATE Set $\bar \theta_j \leftarrow H_j^{-1} g_j$.
    \ENDFOR
    \RETURN Model parameters $(\bar\theta_j)_j$.
\end{algorithmic}
\end{algorithm}

\newcommand\inner[1]{\left\langle #1 \right\rangle}
\begin{algorithm}[H]
\caption{MALA sampling from CSLC distributions}\label{alg:sampling}
\begin{algorithmic}
    \REQUIRE Model parameters $(\bar\theta_j)_j$, an initial sample $x_J$ from $p(x_J)$, step sizes $(\delta_j)_j$, number of steps $(T_j)_j$.
    \FOR{$j = J$ \TO $1$}
        \STATE Initialize $\bar x_{j,0} = 0$.
        \FOR{$t = 1$ \TO $T_j$}
            \STATE Sample $\bar y_{j,t} \sim \mathcal{N}\paren{\bar x_{j,t-1} - \delta_j \nabla_{\bar x_j}\bar E_{\bar\theta_j}(x_j, \bar x_{j,t-1}), 2\delta_j\Id}$.
            \STATE Set $a = \norm{\nabla_{\bar x_j}\bar E_{\bar \theta_j}(x_j, \bar y_{j,t}))}^2 + \norm{\nabla_{\bar x_j}\bar E_{\bar \theta_j}(x_j, \bar x_{j,t-1}))}^2$.
            \STATE Set $b = \inner{ \bar y_{j,t} - \bar x_{j,t-1}, \nabla_{\bar x_j} \bar E_{\bar\theta_j}(x_j, \bar y_{j,t}) - \nabla_{\bar x_j}  \bar E_{\bar\theta_j}(x_j, \bar x_{j,t-1})}$.
            \STATE Set $c= \bar E_{\bar\theta_j}(x_j, \bar y_{j,t}) - \bar E_{\bar\theta_j}(x_j, \bar x_{j,t-1})$.
            \STATE Compute acceptance probability $p = \exp\paren{-\frac{\delta_j}{4}a +\frac12 b - c}$.
            % \STATE Compute acceptance probability $p = \exp\paren{-\frac1{4\delta_j} \paren{\norm{\bar x_{j,t} - (\bar y_{j,t} - \delta_j\nabla_{\bar x_j}\bar E_{\bar \theta_j}(x_j, \bar y_{j,t}))}^2 - \norm{\bar y_{j,t} - (\bar x_{j,t} - \delta_j\nabla_{\bar x_j}\bar E_{\bar \theta_j}(x_j, \bar x_{j,t}))}^2} - \paren{\bar E_{\bar\theta_j}(x_j, \bar y_{j,t}) - \bar E_{\bar\theta_j}(x_j, \bar x_{j,t})}}$.
            \STATE Set $\bar x_{j,t} = \bar y_{j,t}$ with probability $p$ and $\bar x_{j,t} = \bar x_{j,t-1}$ with probability $1 - p$.
        \ENDFOR
        \STATE Reconstruct $x_{j-1} = G_j\trans x_j + \bar G_j\trans \bar x_{j,T_j}$.
    \ENDFOR
    \RETURN a sample $x_0$ from $\hat p_\theta(x)$.
\end{algorithmic}
\end{algorithm}

\section{Experimental Details}
\label{sec:experimental_details}

\subsection{Datasets}

\paragraph{Simulations of $\mathbf{\varphi^4}$.}
We used samples from the $\varphi^4$ model generated using a classical MCMC algorithm, for 3 different temperatures, at the critical temperature $\beta_c \approx 0.68$, above the critical temperature at $\beta = 0.50 < \beta_c$, and  below the critical temperature at $\beta = 0.76 > \beta_c$. For $\beta = 0.76 $, we break the symmetry and only generate samples with positive mean. For each temperature, we generate $10^4$ images of size $128\times 128$.

\paragraph{Weak lensing.}
We used down-sampled versions of the simulated convergence maps from the Columbia Lensing Group \citep[\url{http://columbialensing.org/};][]{PhysRevD.94.083506,PhysRevD.97.103515}. Each map, originally of size $1024\times 1024$, is downsampled twice with local averaging. We then extract random patches of size $128\times 128$.

To pre-process the data, we subtract the minimum of the pixel values over the entire dataset, and then take the square root. This process is reversed after generating samples. We also do not consider the outliers (less than $1\%$ of the dataset) with pixels above a certain cutoff,  in order to reduce the extent of the tail and attenuate weak lensing peaks. Our dataset is made of $\simeq 4\times10^3$ images. 
% Using fewer images deteriorates the quality of the energies for low frequencies (low frequency maps have less pixels than high frequency maps, and therefore the marginals are learned from less points). The data we used was not periodic, as we add to extract patches in order to have enough data.  

\subsection{Experimental Setup}
\label{sec:experimental_setup}

\paragraph{Wavelet filter.}
We used the Daubechies-4 wavelet \citep{doi:10.1137/1.9781611970104}, see the filter in \Cref{plot:Db4Fourier}. 

\paragraph{Wavelet packets.}
We implemented wavelet packets in PyTorch, inspired from the PyWavelets software \citep{Lee2019}. The source code is available at \url{https://github.com/Elempereur/WCRG}.

\paragraph{Score matching.}
We pre-condition the score matching Hessian $H_j$ by normalizing its diagonal before computing $H_j^{-1} g_j$ in \Cref{alg:SM}. After this normalization, we obtain condition numbers $\kappa_{\bar\theta_j}$ which satisfy $\kappa_{\bar \theta_j} \leq 2\times 10^3$ at all $j$.

% \paragraph{Lower Frequency}
% For our experiments, the maps $x_J$ were reduced to a $1$ pixel, we learn its marginal with a scalar potential defined like above. 

\paragraph{Sampling.}
The MALA step sizes $\delta_j$ are adjusted to obtain an optimal acceptance rate of $\approx 0.57$. Depending on the scale $j$, the stationary distribution is reached in $T_j \approx 20$--$400$ iterations from a white noise initialization. We used a qualitative stopping criterion according to the quality of the matching of the histograms and power spectrum.

\subsection{Mixing Times in MALA}
\label{sec:mala_mixing_time}

Sampling from $p_{\theta}$ requires sampling from $p_{\theta_J}$, and then conditionally sampling from $p_{\bar{\theta}_j}(\bar{x}_j | x_j)$. This last step is performed with a Markov chain whose stationary distribution is $p_{\bar{\theta}_j}(\bar{x}_j|x_j)$ for a given $x_j$. It generates successive samples $\bar{x}_j(t)$ where $t$ is the the step number in the Markov chain.

% We also introduce, for a given $x_j$:
% $$\delta \bar{x}_j(t) = \bar{x}_j(t) - \mathbb{E}[\bar{x}_j\mid x_j], $$
% with $\mathbb{E}$  an expectation taken with respect the sampled $\bar{x}_j$. 
We introduce the conditional auto-correlation function:
$$A_{j}(t) = \frac{\mathbb{E} \left[ \paren{\bar{x}_j(t) - \mathbb{E}[\bar{x}_j \,|\, x_j]} \paren{\bar{x}_j(0) - \mathbb{E}[\bar{x}_j \,|\, x_j]} \right]}{\mathbb{E}[\delta \bar{x}_j^2]}.$$
The expected value $\mathbb{E}$  is taken with respect to both $x_j$ and the sampled $\bar{x}_j$.
$A_j(t)$ has an exponential decay.
Let $\bar{\tau}_j$ be the mixing time defined
as the time it takes for the Markov chain to generate two independent samples:
$$A_{j}(t) \approx A_j(0) \exp\paren{-\frac{t}{\bar{\tau}_j}}.$$
$\bar{\tau}_j$ is computed by regressing $\log(A_{j}(t))$ over $t$.

Each iteration of MALA with $p_{\bar{\theta}_j}(\bar{x}_j\mid x_j)$  computes a gradient of size $\bar{d}_j$. In order to estimate the real computational cost of the sampling of $p_{\theta}$, we average $\bar \tau_j$ proportionally to the dimension $\bar d_j$:
$$\bar{\tau} = \sum\limits_{j=1}^{J} \frac{{\bar{d}_j}}{d}\bar{\tau}_j +{\tau}_J\frac{d_J}{d},$$
where $d$ is the dimension of $x$.

% Since $\bar{\tau_j}$ increases, and since $d = d_J+\sum\limits_{j=1}^J \bar d_j $, the only relevant sampling is the finer scale.

% At phase transition, $\varphi^4$ sampling suffers from critical slowing down. Sampling a map $x$ of size $\times L$ from its energy $E(x)$ leads to a $\tau^D_{MC}$ 

% $$\tau^D_{MC} \sim L^z$$

% With $z\approx 2$

% As we show figure \ref{plot:Mix}, our algorithm does not suffers from critical slowing down

% $$\tau_{MC} \sim L$$

% This scaling also hold if we generate from $x_J$ and iterate to obtain $x_0$, since $L+2^{-1}L+...+2^{-J}L \sim L $

\section{Energy Estimation with Free-Energy Modeling}
\label{sec:energy_estimation}

This section explains how to recover an explicit parametrization of the negative log-likelihood $-\log p_\theta$ from the parameterized energies $\bar E_{\bar \theta_j}$. We introduce a parameterization of the normalization constant of the Gibbs energies 
for each $j$ and describe an efficient score-matching algorithm to learn the parameters. This leads to a decomposition of the negative log-likelihood $-\log p_\theta$ over scales. 

\subsection{Free-Energy Score Matching}

% A parametric
% model of $p(x)$ is obtained 
% from Gibbs parametric models of $p (x_J)$ and
% $p(x_j | \bar x_j)$, defined by
% \begin{equation}
% \label{condexp0}
% p_{\theta_J} (x_j) = Z_J^{-1} e^{-E_{\theta_J}(x_j)}
% \end{equation}
% with $E_{\theta_J}(0) = 0$ and
% \begin{equation}
% \label{condexp2}
% p_{\bar \theta_j} (\bar x_j | x_j) = \bar Z^{-1}_{j}(x_j)\, e^{-\bar E_{\bar \theta_j}(\bar x_j, x_j)}
% \end{equation}
% with $\bar E_{\bar \theta_j}(x_j,0) = 0$ and
% \begin{equation}
% \label{condexp3}
%   \bar Z_{j}(x_j) = \int e^{-\bar E_{\bar \theta_j}(\bar x_j, x_j)}\, d \bar x_j .
% \end{equation}

% Inserting each of these models in (\ref{condexp}) defines
% a parametric model of $p(x)$
% \begin{equation}
% \label{condexp2}
% p_\theta (x) = p_{\theta_J} (x_J)\, \prod_{j=1}^J p_{\bar \theta_j} (\bar{x}_j |x_j) 
% \end{equation}
% where $\theta = (\theta_J , \bar \theta_j)_{1 \leq j \leq J}$.
% A marginally log-concave model is obtained if $E_{\theta_J} (x_J)$ is convex and if each $\bar E_{\bar \theta_j}(\bar x_j, x_j)$ is a convex function of $\bar x_j$ for $x_j$ fixed.

% Imposing that  $p_{\theta_{j-1}} (x_{j-1}) = 
% p_{\theta_j} (x_j)\, p_{\bar \theta_j} (\bar x_j | x_j)$
% is equivalent to impose that
% \begin{equation}
% \label{relation}
% E_{\theta_{j-1}}(x_{j-1}) = E_{\theta_j}(x_{j}) + \bar E_{\bar \theta_j}(\bar{x}_j,x_{j}) + \log \bar Z_{j}(x_{j}) + c_j
% \end{equation}
% where $c_j$ is a constant.

From the decomposition
\begin{equation*}
    p_\theta(x)  = p_{\theta_J} (x_J) \prod_{j=1}^J p_{\bar \theta_j} (\bar{x}_j |x_j),
\end{equation*}
we obtain
\begin{equation}
    \label{eq:global_energy_decomposition}
    -\log p_\theta(x) = E_{\theta_J}(x_J) + \sum_{j=1}^J \paren{\bar E_{\bar\theta_j}(x_j, \bar x_j) + \log \bar Z_{\bar\theta_j}(x_j)} + \mathrm{cst},
\end{equation}
where $\bar Z_{\bar\theta_j}(x_j)$ is the normalization constant for $\bar E_{\bar\theta_j}(x_j, \bar x_j)$. To retrieve the global negative log-likelihood $-\log p_\theta(x)$, we thus compute an approximation of $-\log \bar Z_{\bar\theta_j} (x_j)$ with a parametric family $F_{\tilde\theta_j}$.
% \begin{equation}
%     \label{eq:exp_model_logz}
%     F_{\tilde\theta_j} = -\log \bar Z_{\tilde \theta_j}  = \tilde\theta_j\trans\tilde{\Phi}_j.
% \end{equation}

% The linear models are defined over domains
% $\tilde\Theta_j$ such that $\int_{\mathbb{R}^{d_j}} \exp\left(-\tilde\theta_j\trans\tilde{\Phi}_j(x_j)\right) dx_j < \infty$.

The parameters $\tilde \theta_j$ of the approximation of the normalizing factors $\bar Z_{\tilde \theta_j}$ can be learned in a manner similar to denoising score matching. Indeed, using the identity
\begin{equation*}
    -\nabla_{x_j} \log \bar Z_{\bar\theta_j}(x_j) = \expect{\nabla_{x_j} \bar E_{\bar\theta_j}(x_j, \bar x_j) \,|\, x_j},
\end{equation*}
% This conditional expectation can be estimated for each $x_j$ by generating several $\bar x_j$ from $p(\bar x_j | x_j)$, which can be done efficiently if $p(\bar x_j | x_j)$ is log-concave.
which can be proven by a direct computation of the gradient, the parameters $\tilde \theta_j$ can be estimated by minimizing % $\expect{\norm{\nabla_xF - \expect{\nabla_{x_j} \bar E_{\bar\theta_j}(x, \bar x) \,|\, x}}^2}$, and therefore by minimising
% \begin{equation}
%     \label{eq:free_energy_regression1}
%     \tilde \ell_j(\tilde \theta_j) = \expect{\norm{-\nabla_{x_j}\log Z_{\tilde \theta_j} - \expectt{\nabla_{x_j} \bar E_j \,|\, x_j}}^2}.
% \end{equation}
% An alternative, simpler algorithm consists in exploiting the properties of the conditional expectation to directly minimize the equivalent objective:
\begin{equation}
    \label{eq:free_energy_regression2}
    \tilde \ell_j(\tilde \theta_j) = \expect{\norm{\nabla_{x_j} F_{\tilde\theta_j} - \nabla_{x_j} \bar E_{\bar{\theta}_j}}^2}.
\end{equation}

For an exponential model $F_{\tilde\theta_j} = \tilde\theta_j\trans\tilde{\Phi}_j$ with a fixed potential vector $\tilde \Phi_j$, \cref{eq:free_energy_regression2} is quadratic in $\tilde \theta$ and admits a closed-form solution:
\begin{equation*}
   % \label{eqn:freecloseform}
    \tilde\theta_j = \expect{\nabla_{x_j} \tilde\Phi_j\nabla_{x_j} \tilde\Phi_j\trans}^{-1} \expect{\nabla_{x_j} \tilde\Phi_j\nabla_{x_j} \bar E_{\bar{\theta}_j}}.
\end{equation*}
We finally obtain the energy decomposition
\begin{equation}
    \label{eq:global_energy_decomposition2}
    -\log p_\theta(x) = E_{\theta_J}(x_J) + \sum_{j=1}^J \paren{\bar E_{\bar\theta_j}(x_j, \bar x_j) - F_{\tilde\theta_j}(x_j)} + \mathrm{cst}.
\end{equation}
This score-based method is much faster and simpler to implement than likelihood-based methods such as the thermodynamic integration of \citet{marchand_wavelet_2022}, which requires generation of many samples while varying the parameters $\bar \theta_j$ of the conditional energy $\bar E_{\bar \theta_j}$.

\subsection{Parameterized Free-Energy Models}

The potential vector $\tilde \Phi_j$ is modeled in the class of \cref{scal-en}, following \citet{marchand_wavelet_2022} and similarly to \Cref{app:phibar}:
\begin{align*}
   % \label{eqn:Zjdix}
    F_{\tilde{\theta}_j}(x_j) &= \frac{1}{2}x_j\trans\tilde{K}_jx_j + \tilde{V}_j(x_j) + \sum\limits_i\tilde{v}_j(x_j[i]) \\
    \tilde{v}_j(t) &= \sum\limits_k \tilde{\alpha}_{j,k} \tilde{\rho}_{j,k}(t),
\end{align*}
which gives $\tilde{\theta}_j =(\tilde{K}_j,\tilde{\alpha}_{j,k})_k$ and an associated potential vector
\begin{equation*}
    \tilde{\Phi}_j(x_j) = \paren{\frac{1}{2}x_jx_j\trans,\tilde{\rho}_{j,k}(x_j)}_{k}.
\end{equation*}

\subsection{Multiscale Energy Decomposition}

We now expand the models for the conditional energies $\bar E_{\bar\theta_j}$ and the so-called free energies $F_{\tilde\theta_j}$ in \cref{eq:global_energy_decomposition2}. All the quadratic terms $(K_J,\bar K_j, \tilde K_j)_j$ can be regrouped in an equivalent quadratic term $K$. We then have
\begin{align*}
-\log p_\theta(x) &= \frac12 x\trans K x +\sum\limits_i \left[ v_{J}(x_J[i])+ \sum_{j=1}^{J}  \paren{\bar{v}_{j}(x_{j-1}[i]) - \tilde{v}_{j}(x_{j}[i])} \right]\\
    &= \frac12 x\trans K x + \sum\limits_i \left[ \bar v_1(x_0[i]) + \sum_{j=1}^{J} \paren{\bar v_{j+1}(x_j[i]) - \tilde v_j(x_j[i])}\right] ,
\end{align*}
with $\bar v_{J+1} = v_J$. This defines multiscale scalar potentials $V_j$: 
$$ \begin{array}{l}
            V_j = \bar{v}_{j+1} -\tilde{v}_{j},\\
            V_0 = \bar v_1,
    \end{array} $$
such that we have the global negative log-likelihood or energy function:
\begin{equation*}
    -\log p_\theta(x) = \frac12 x\trans K x +  \sum_{j=0}^{J} \sum_i V_j(x_j[i]).
\end{equation*}
For $\varphi^4$ at critical temperature, as derived in \cite{marchand_wavelet_2022}, the only non-zero scalar potential will be $V_0$. The other $V_j$ potentials are zero, up to a quadratic term.

As a numerical test, \Cref{plot:FreePhi4} verifies that on $\varphi^4$ at critical temperature, $\bar v_{j+1}$ and $\tilde v_j$ indeed cancel out so that $V_j = 0$ for $j > 0$. In order to ensure that the quadratic difference mentioned above vanishes, we subtract to $\tilde v_j$ the quadratic interpolation of $\tilde v_j-\bar v_{j+1}$.

% As a numerical test, \Cref{plot:FreePhi4} verifies that we recover the double-well potential $v(t) = t^4 - (1 + 2\beta) t^2$ of the $\varphi^4$ scalar field. 

\begin{figure}[t]
\centering
\includegraphics[width=0.45\textwidth]{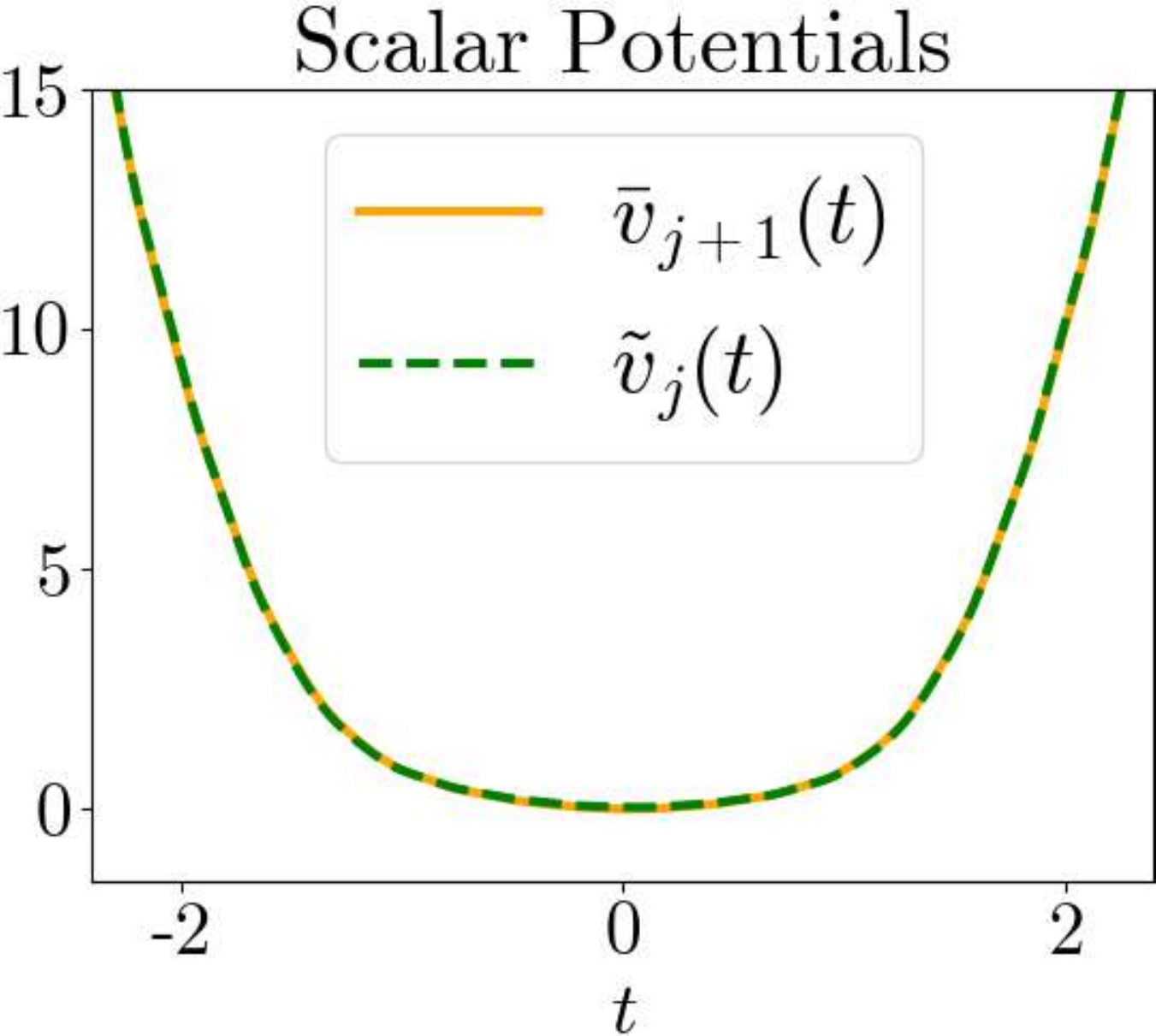}
\caption{For $\varphi^4$ at $\beta_c$, the conditional potentials $\bar v_{j+1}$ and free-energy potential $\tilde v_j$ cancel out. Only $j=1$ is shown, other scales show similar behavior. }
\label{plot:FreePhi4}
\end{figure}

% {\color{blue}J Just give $c_1$ constant formula}
% $c_1$ is a normalisation constant, whose interest appears in theorem \ref{th:scalar}.

% As in WCRG, and adapting \ref{}, we model the energies with a scalar and a quadratic term.

% Where $V_j =\sum \alpha_j^kv_k$ and $v_k(x) = \sum\limits_{n} \rho_k(x(n))$

% Naturally,

% \begin{equation}
%    \label{eqn:ansatzJ}
%     \left
%         \Phi_J(x_J) = (\frac{1}{2}x_Jx_J\trans,v_k (x_J))_{k}
%     \right.
% \end{equation}

% \begin{equation}
%    \label{eqn:ansatz}
%     \left
%         \bar{\Phi}_j(x_j,\bar{x_j}) = (\frac{1}{2}\bar{x}_j \bar{x}_j\trans,\bar{x}_j \bar{x}_{j+l}\trans,v_k (x_{j-1}))_{l,k}
%     \right.
% \end{equation}

% Using shift invariance,

% \begin{equation}
%    \label{eqn:scalar}
%     \left
%        v_k(x) = \sum\limits_{n} \rho_k(x(n))
%     \right.
%\end{equation}

\input{proofs_sec3.tex}

\input{proofs_scalar_potential.tex}

%% file: proofs_sec3.tex
\section{Proofs of Section \ref{sec:section3}}
\label{sec:proofs_sec3}

\subsection{Proof of Proposition \ref{prop:global_error}}
\errordecomp*
\begin{proof}
We use the following decomposition of KL divergence in terms of conditional distributions: 
\begin{lemma}
\label{lem:klcond}
Let $p(x) = p(x_J) \prod_{j=1}^J p(\bar{x}_j| x_j)$ and $q(x) = q(x_J) \prod_{j} q(\bar{x}_j|x_j)$. 
We have 
$\mathrm{KL}(p \,\|\, q) = \sum_j \mathbb{E}_{x_j\sim p} \mathrm{KL}( p( \cdot | x_j) \,\|\, q( \cdot | x_j) )~.$
\end{lemma}

Using Lemma \ref{lem:klcond} we obtain that 
$\mathrm{KL}( p \,\|\, p_\theta) = \epsilon_J^L + \sum_j \bar{\epsilon}^L_j$ 
and $\mathrm{KL}( \hat{p} \,\|\, p_\theta) = \epsilon_J^S + \sum_j \bar{\epsilon}^LSj$. 
We conclude with the Pinsker inequality
$$\mathrm{TV}( \hat{p}, p) \leq \mathrm{TV}( \hat{p}, p_\theta) + \mathrm{TV}( {p}_\theta, p) \leq \frac{1}{\sqrt{2}} \left(\sqrt{\mathrm{KL}( p \,\|\, p_\theta)} + \sqrt{\mathrm{KL}( \hat{p} \,\|\, p_\theta)}\right)~.$$
\end{proof}

\begin{proof}[Proof of Lemma \ref{lem:klcond}]
    We proceed by induction over $J$. Observe that 
    $\log p(x) = \log p( \bar{x}_1| x_1) + \log p(x_1)$, so 
    \begin{align*}
   \mathrm{KL}(p \,\|\, q) &= \mathbb{E}_p [\log(p) - \log(q)] \\
   &= \mathbb{E}_p[\log p(x_1) - \log q(x_1)] + \\
   & \mathbb{E}_{x_1 \sim p(x_1)} \mathbb{E}_{\bar{x}_1 \sim p(\bar{x}_1 | x_1)} [\log p( \bar{x}_1| x_1) - \log q( \bar{x}_1| x_1)] \\
   &= \mathrm{KL}( p(x_1) \,\|\, q(x_1) ) \\
   &+ \mathbb{E}_{x_1 \sim p(x_1)} \mathrm{KL}( p( \cdot | x_1) \,\|\, q( \cdot | x_1) )~.
    \end{align*}
The first term $\mathrm{KL}( p(x_1) \,\|\, q(x_1) )$ now involves $J-1$ factors, and hence we can apply the induction step to conclude.
 \end{proof}

\subsection{Proof of Theorem \ref{prop:learn2}}

We will use a general concentration result of the empirical covariance for general distributions with mild moment assumptions \cite{vershynin2018high}, as well 
as anticoncentration properties of the random design \cite{mourtada2022exact}. 
Together, they provide enough control on the probability tails so that the inverse covariance concentrates to the precision matrix in expectation. 
% \begin{assumption}
% \label{ass:normbound}
% Let $X \in \mathbb{R}^m$ be a random vector. Assume that 
%  there exists $K\geq 1$ such that $\|X\|_2 \leq K \mathbb{E}[\|X\|_2^2]^{1/2}$ almost surely.   
% \end{assumption}
\begin{assumption}
\label{ass:normbound}
Let $X \in \mathbb{R}^{m \times d}$ be a random matrix. Assume that 
 there exists $K\geq 1$ such that $\|X\|_F \leq K \mathbb{E}[\|X\|_F^2]^{1/2}$ almost surely.   
\end{assumption}

\begin{theorem}[General Covariance Estimation with High Probability, {\citep[Theorem 5.6.1, Ex 5.6.4]{vershynin2018high}}]
\label{thm:covconc} 
Let $X \in \mathbb{R}^{m \times d}$ be a random matrix satisfying assumption \ref{ass:normbound}. Let $\Sigma = \mathbb{E}[X X^\top]$, and for any $n$ let $\hat{\Sigma}_n =\frac1n \sum_i X_i X_i^\top$ be the sample covariance matrix, where $X_i$ are $n$ iid copies of $X$. 
There exists an absolute constant $C$ such that for any $\delta>0$, it holds
\begin{equation}
\label{eq:covproba}
    \| \hat{\Sigma}_n - \Sigma \| \leq C \left(\sqrt{\frac{K^2 m (\log( m) + \log(2/\delta))}{n}} + \frac{K^2 m (\log( m) + \log(2/\delta))}{n}  \right) \| \Sigma\|
\end{equation}
with probability at least $1-\delta$. 
\end{theorem}

\begin{assumption}[Moment Condition]
\label{ass:normbound_our}
% Assume that there exists $K$ such that $X:= (\nabla_{\bar{x}} \bar{\Phi}_k(\bar{x},x))_{k \leq m} \in \mathbb{R}^{m \times d}$, with $(\bar{x},x)\sim p(\bar{x},x)$ satisfies Assumption \ref{ass:normbound} with constant $K$ for all $x$. 
% Furthermore, suppose that $Y = (\Delta_{\bar{x}} \bar{\Phi}_k(\bar{x},x))_{k \leq m} \in \mathbb{R}^m$ with $(\bar{x},x)\sim p(\bar{x},x)$ has finite fourth moment, with $\mathbb{E}[\|Y\|^4] \leq \xi \mathbb{E}[\|Y\|^2]^2$.  
Assume that there exists $K_X$ and $K_Y$ such that $X:= (\nabla_{\bar{x}} \bar{\Phi}_k(\bar{x},x))_{k \leq m} \in \mathbb{R}^{m \times d}$, and $Y = (\Delta_{\bar{x}} \bar{\Phi}_k(\bar{x},x))_{k \leq m} \in \mathbb{R}^m$ satisfy Assumption \ref{ass:normbound} with constants $K_X$ and $K_Y$ respectively, where  $(\bar{x},x)\sim p(\bar{x},x)$. 
\end{assumption}

\begin{assumption}[Anticoncentration Condition, {\citep[Assumption 1]{mourtada2022exact}}]
\label{ass:smallball}
    The random matrix $X=(\nabla_{\bar{x}} \bar{\Phi}_k(\bar{x},x))_{k \leq m} \in \mathbb{R}^{m \times d}$ 
    satisfies the following: there exists constants $C\geq 1$ and $\nu \in (0,1]$ such that for every $\theta \in \mathbb{R}^m \setminus \{0\}$ and $t>0$, $\mathbb{P}( \theta^\top {X} {X}^\top \theta \leq t^2 \theta^\top \mathbb{E}[XX^\top] \theta ) \leq (C t)^\nu$. 
\end{assumption}

\begin{theorem}[{\citep[Corollary 3]{mourtada2022exact}}]
\label{thm:invertible_cov}
Let $X\in \mathbb{R}^{m \times d}$ be a random matrix satisfying Assumption \ref{ass:smallball} and such that $\mathbb{E}[\| X\|_F^2]<\infty$, with $\Sigma = \mathbb{E}[X X^\top]$. Then, if $m/n \leq \nu /6$, for every $t\in(0,1)$, the empirical covariance matrix $\hat{\Sigma}_n$ obtained from an iid sample of size $n$ satisfies 
$$\hat{\Sigma}_n \succeq t \Sigma$$
with probability with probability greater than $1 - (\tilde{C} t)^{\nu n/6}$, where $\tilde{C}$ only depends on $C$ and $\nu$ in Assumption \ref{ass:smallball}. 
\end{theorem}

% \begin{assumption}
% \label{ass:normbound_our}
% For $l=1\ldots d$, assume that there exists $K$ such that $X_l:= (\nabla_{\bar{x}_l} \bar{\Phi}_k(\bar{x},x))_{k \leq m}$, with $(\bar{x},x)\sim p(\bar{x}|x)$ satisfies Assumption \ref{ass:normbound} for $l=1\ldots d$ with constant $K$. 
% Furthermore, suppose that $Y = (\Delta \bar{\Phi}_k(\bar{x}|x))_{k \leq m} \in \mathbb{R}^m$ with $(\bar{x},x)\sim p(\bar{x}|x)$ has finite second moment. 
% \end{assumption}

%\SMexpo*

\begin{restatable}[Excess risk for CSLC exponential models, Theorem \ref{prop:learn2} restated]{theorem}{SMexpo2}
\label{prop:learn3}
    Let ${\bar\theta}^\star=\arg\min \ell({\bar\theta})$ and $\hat{\bar\theta} = \arg\min \hat\ell(\bar\theta)$. Assume:
    \begin{enumerate}[label={(\roman*)}]
    %(i) ${\bar\theta}^\star=\arg\min_{{\bar\theta} \in \mathbb{R}^m} \mathbb{E}_{p} \mathrm{FI}( p(\bar x|x) \,\|\, p_{\bar\theta}(\bar x | x)) \in \Theta_\alpha$ for some $\alpha>0$, 
    \item $\bar\theta^\star \in \Theta_{\bar\alpha}$ for some ${\bar\alpha}>0$, 
    \item ${H}=\expect{\nabla_{\bar x} \bar\Phi \nabla_{\bar x} \bar\Phi\trans} \succeq \eta \Id$ with $\eta>0$,
    \item the sufficient statistics $\bar{\Phi}$ satisfy moment conditions \ref{ass:normbound_our}, regularity conditions \ref{ass:smallball}, and $\nabla \bar{\Phi}_k(x,\bar x)$ is $M_{\bar{\Phi}}$-Lipschitz for any $k\leq m$ and all $x$.
%    for each $i \in [\bar{d}]$, the random vector $X_i:=(\partial_{\bar{x}_i} \bar\Phi_k(x,\bar{x}))_{k=1\ldots m} \in \mathbb{R}^m$ satisfies the small-ball condition \cite[Assumption 1]{mourtada19} and the $L^2-L^4$ norm equivalence \cite[Assumption 3]{mourtada19}, 
    \end{enumerate}
    %There exists constants $\tilde{c},c,C$ depending only on $\bar{\Phi}$ and $p$ such that
    Then when $n > m$, the empirical risk minimizer $\hat{\bar\theta}$ satisfies:
    \begin{equation}
    \label{eq:scprop}
        \hat{\bar\theta} \in \Theta_{\hat{{\bar\alpha}}}\ \text{with } \expect[(\bar x^{i}, x^{i})]{\hat{{\bar\alpha}}} \geq {\bar\alpha} - O\left(\eta^{-1} \sqrt{\frac{m}{n}}\right),
    \end{equation}
     \begin{equation}\label{eq:FIgen}
        \expect[(\bar x^{i}, x^{i})]{\ell(\hat{\bar{\theta}})} \leq \left[\ell({\bar\theta}^\star) + O\left(\kappa(H) \eta^{-1} \frac{m}{n} \right)\right], %\expect[x]{\mathrm{KL}(p(\bar x|x)\,\|\,p_{\hat{{\bar\theta}}}(\bar x|x))}
    %     %\expect[(\bar x^{(i)}, x^{(i)})]{\bar\epsilon^L} \leq \frac{1}{\hat{\alpha}}\left[\ell({\bar\theta}^\star) + O\left(\kappa(H) \eta^{-1} \frac{m}{n} \right)\right], %\expect[x]{\mathrm{KL}(p(\bar x|x)\,\|\,p_{\hat{{\bar\theta}}}(\bar x|x))}
     \end{equation}
    and, for %$t \ll \sqrt{m} M_\Phi \ell(\bar{\theta}^\star) \|H\|^{-1}$,
    $t \ll \sqrt{m} \ell(\bar{\theta}^\star)$,
    \begin{equation}\label{eq:KLgen}
       \bar{\epsilon}^L \leq \frac{\ell({\bar\theta}^\star)}{{\bar\alpha}}(1+t)  %\expect[x]{\mathrm{KL}(p(\bar x|x)\,\|\,p_{\hat{{\bar\theta}}}(\bar x|x))}
        %\expect[(\bar x^{(i)}, x^{(i)})]{\bar\epsilon^L} \leq \frac{1}{\hat{\alpha}}\left[\ell({\bar\theta}^\star) + O\left(\kappa(H) \eta^{-1} \frac{m}{n} \right)\right], %\expect[x]{\mathrm{KL}(p(\bar x|x)\,\|\,p_{\hat{{\bar\theta}}}(\bar x|x))}
    \end{equation}
    with probability greater than $1 - \exp\left\{ -O(n \log (tn/\sqrt{m}))\right\} $ over the draw of the training data. 
    % where $\kappa(H) = \|H\|/\eta$,
    % \begin{equation}
    % \label{eq:scprop}
    %     \hat{\bar\theta} \in \Theta_{\hat{\alpha}}\ \text{with } \mathbb{E}_{(\bar x^{(i)}, x^{(i)})}\hat{\alpha}\geq \alpha - O\left(\eta^{-1} \sqrt{\frac{m}{n}}\right),
    % \end{equation}
    %  and in expectation over the draw of the training data:
    % %$$\mathbb{E}_{x} \mathrm{FI}( p(\bar x|x) \,\|\, p_{\hat{{\bar\theta}}}(\bar x | x)) \leq \mathbb{E}_{x} \mathrm{FI}( p(\bar x|x) \,\|\, p_{{\bar\theta}^\star}(\bar x | x)) + {Cm/n}~.$$
    % % \begin{equation}
    % %    \mathbb{E}_{x} \mathrm{KL}( p(\bar x|x) \,\|\, p_{\hat{{\bar\theta}}}(\bar x | x)) \leq \mathbb{E}_{x} \mathrm{FI}( p(\bar x|x) \,\|\, p_{{\bar\theta}^\star}(\bar x | x)) +   O\left(\frac{\eta^{-1} \kappa(H) m}{n} \right)~,
    % %    \end{equation}
    %    %  \begin{equation}
    %    % \mathbb{E} \ell(\hat{{\bar\theta}})  \leq \ell({\bar\theta}^\star) + O\left(\frac{\eta^{-1} \kappa(H) m}{n} \right)~,
    %    % \end{equation}
    % \begin{equation}\label{eq:klgen}
    %    \expect[(\bar x^{(i)}, x^{(i)})]{\ell(\hat{\bar{\theta}})} \leq \left[\ell({\bar\theta}^\star) + O\left(\kappa(H) \eta^{-1} \frac{m}{n} \right)\right], %\expect[x]{\mathrm{KL}(p(\bar x|x)\,\|\,p_{\hat{{\bar\theta}}}(\bar x|x))}
    %     %\expect[(\bar x^{(i)}, x^{(i)})]{\bar\epsilon^L} \leq \frac{1}{\hat{\alpha}}\left[\ell({\bar\theta}^\star) + O\left(\kappa(H) \eta^{-1} \frac{m}{n} \right)\right], %\expect[x]{\mathrm{KL}(p(\bar x|x)\,\|\,p_{\hat{{\bar\theta}}}(\bar x|x))}
    % \end{equation}
   The constants in $O(\cdot)$ only depend on moment and regularity properties of $\bar{\Phi}$.
\end{restatable}

%\end{theorem}
\begin{proof}
We can rewrite the score-matching population risk in terms of a joint distribution $(X,Y) \in \mathbb{R}^{m \times d} \times \mathbb{R}^m$:
$$\min_{\theta} \ell(\theta)=\mathbb{E}_{(X,Y)} \left[\frac12 \theta^\top X X^\top \theta - \theta^\top Y\right] = \frac12 \theta^\top H \theta - \theta^\top g~,$$
where $H = \mathbb{E}[XX^\top]$ and $g = \mathbb{E}[Y]$. The empirical objective is the quadratic form 
\begin{equation}
\min_{\theta} \frac12 \theta^\top \hat{H} \theta - \theta^\top \hat{g}~,   
\end{equation}
with $\hat{H}=\frac1n \sum_{i=1}^n X_i X_i^\top$ and $\hat{g} = \frac1n \sum_i Y_i$. 

We want to control the expected excess risk $\mathbb{E} \ell( \hat{\theta}) - \ell( \theta^*)$ and the norm $\| \hat{\theta} - \theta^*\|$
, where 
$$\hat{\theta} = \hat{H}^{-1} \hat{g}~,~\theta^* = H^{-1} g~.$$
Since $\ell(\theta)$ is quadratic and $\theta^*$ is its global minimum, 
observe that 
\begin{align}
\label{eq:gruy}
\ell(\theta) - \ell(\theta^*) &= \nabla_\theta \ell(\theta^*)^\top( \theta - \theta^*) + \frac12 (\theta - \theta^*)^\top \nabla_\theta^2 \ell(\theta^*) (\theta - \theta^*) \nonumber \\
&= \frac12  (\theta - \theta^*)^\top H (\theta - \theta^*) ~,
\end{align}
which shows that the excess risk can be bounded from the mean-squared error $\mathbb{E}\| \hat{\theta} - \theta^*\|^2$ with 
\begin{equation}
\label{eq:basic}
\mathbb{E} \ell(\hat{\theta}) - \ell(\theta^*) \leq \frac{\|H\|}{2}\mathbb{E} \|\hat{\theta} - \theta^*\|^2~.    
\end{equation}
Let $\upsilon := \hat{g} - g$ and $\Upsilon = \hat{H}^{-1} - {H}^{-1}$. 
By definition, we have
\begin{equation}
\label{eq:deftheta}
\hat{\theta} - \theta^* = \hat{H}^{-1} (g + \upsilon) - H^{-1} g = \Upsilon g + \hat{H}^{-1} \upsilon~,
\end{equation}
so 
\begin{equation}
\label{eq:basic2}
   \mathbb{E} \| \hat{\theta} - \theta^*  \|^2 \leq 2(\mathbb{E}\| \Upsilon \|^2) \|g\|^2 + 2\mathbb{E} \| \hat{H}^{-1} \upsilon \|^2~.
\end{equation}

% Let $\widetilde{H} = \hat{H}^{-1} H \hat{H}^{-1}$ and define the residuals
% $$\upsilon = \hat{g} - g~,~\Upsilon_1 = \widetilde{H} - {H}^{-1}~,~\Upsilon_2 = \hat{H}^{-1} - {H}^{-1}~.$$
% Note that $\mathbb{E}\upsilon = 0$, and thus 
% \begin{align*}
%   \mathbb{E} \ell( \hat{\theta}) - \ell( \theta^*) =  \mathbb{E} & \frac12 [\hat{g}^\top \widetilde{H} \hat{g} - g^\top K^{-1} g] 
%    + g^\top(H^{-1} g - \hat{H}^{-1} \hat{g}) \\
% = &\frac12 g^\top H^{-1} g - g^\top ( {H}^{-1} + \Upsilon_2) g - g^\top ( H^{-1} + \Upsilon_2) \upsilon  \\
% & + \frac12 \left(g^\top \tilde{H} g + \upsilon^\top \tilde{H} \upsilon + g^\top \tilde{H} \upsilon + \upsilon^\top \tilde{H} g\right) \\
% =& g^\top \Upsilon_2 g - g^\top \Upsilon_2 \upsilon  + \frac12 \left(g^\top \Upsilon_1 g + \upsilon^\top \tilde{H} \upsilon + g^\top \Upsilon_1 \upsilon + \upsilon^\top \Upsilon_1 g \right) \\
% \leq & \left\| \Upsilon_2 + \frac12 \Upsilon_1\right \| \|g\|^2 + o(n^{-1})
% \end{align*}

Let us begin with the first term in the RHS of (\ref{eq:basic2}), involving $\Upsilon$. We claim that there exists $C_0$, only depending on the assumption parameters in \ref{ass:normbound_our} and \ref{ass:smallball}, such that
\begin{equation}
\label{eq:ups2}
    \mathbb{E} \| \Upsilon \|^2 \leq C_0 \frac{\| H^{-2}\|}{n} + O\left( \frac{m^3}{n^2}\right)~.
\end{equation}
The main technical ingredient is to exploit upper and lower tail bounds of $\hat{H}=\hat{H}_n$ to establish 
a control on expectation, via the following Lemma. 
\begin{lemma}[From tail bounds to Expectation]
\label{lem:tails_to_exp}
Suppose the empirical covariance $\hat{\Sigma}_n$ satisfies the following lower and upper tail bounds:
\begin{align}
\hat{\Sigma}_n \preceq (1+s) \Sigma & ~\text{ with probability greater than }~1-\eta_n(s)~,s\geq 0,\nonumber \\
\hat{\Sigma}_n \succeq (1-t) \Sigma & ~\text{ with probability greater than }~1-\delta_n(t)~, \, t\in (0,1)~.
\end{align}
Then 
\begin{eqnarray}
    \mathbb{E} \| \hat{\Sigma}_n^{-1} - \Sigma^{-1} \| &\leq& \|\Sigma^{-1} \| \left(\int_0^\infty \delta_n\left(\frac{\beta}{1+\beta}\right) d\beta + \int_0^1 \eta_n\left(\frac{\beta}{1-\beta}\right) d\beta \right)~, \label{eq:firstmom} \\
        \mathbb{E} \| \hat{\Sigma}_n^{-1} - \Sigma^{-1} \|^2 &\leq& \|\Sigma^{-1} \|^2 \left(\int_0^\infty \beta \delta_n\left(\frac{\beta}{1+\beta}\right) d\beta + \int_0^1 \beta \eta_n\left(\frac{\beta}{1-\beta}\right) d\beta \right)~. \label{eq:secmom}
\end{eqnarray}
% and 
% \begin{equation}
%     \mathbb{E} \| \hat{\Sigma}_n^{-2} - \Sigma^{-2} \| \leq \|\Sigma^{-2} \| \left(\int_0^\infty \delta_n\left(\frac{\beta}{1+\beta}\right) d\beta + \int_0^1 \eta_n\left(\frac{\beta}{1-\beta}\right) d\beta \right)~,
% \end{equation}
\end{lemma}
Thanks to assumptions \ref{ass:smallball} and \ref{ass:normbound}, the tail bounds of Theorems \ref{thm:covconc} and \ref{thm:invertible_cov} apply, yielding 
\begin{equation}
\delta_n(t) = \min((\tilde{C}(1-t))^{\nu n/6}, 2m \exp( -n^2 t^2 / C m))~,~\eta_n(s) = 2m \exp( -n^2 s^2 / C m)~.
\end{equation}
We now apply Lemma \ref{lem:tails_to_exp} with these values. 
Let us first address the term $\eta_n$. We have 
$$\eta_n( \beta/(1-\beta)) = 2m \exp( -n^2 \beta^2 (1-\beta)^{-2} /(Cm))~,$$
and hence
\begin{align}
\label{eq:d11}
%\int_0^1 \eta_n( \beta/(1-\beta)) d\beta &= 2m \int_0^1 \exp( -n^2 \beta^2 (1-\beta)^{-2} /(Cm)) d\beta \nonumber \\
%&\leq 2m \int_0^1 \exp( -n^2 \beta^2 /(Cm)) d\beta \nonumber \\
%&\leq  \sqrt{2\pi C} \frac{m^{3/2}}{n} ~.
\int_0^1 \beta \eta_n( \beta/(1-\beta)) d\beta &= 2m \int_0^1 \beta \exp( -n^2 \beta^2 (1-\beta)^{-2} /(Cm)) d\beta \nonumber \\
&\leq 2m \int_0^1 \beta \exp( -n^2 \beta^2 /(Cm)) d\beta \nonumber \\
& \leq 2m \sqrt{\pi}\frac{\sqrt{C m/2}}{n} \mathbb{E}_{Z \sim \mathcal{N}(0, Cm/(2n^2))}[|Z|] \\
&\leq  \widetilde{C} \frac{m^3}{n^2} ~.
\end{align}
Let us now study the term in $\delta_n$. For any $\beta^*$ we have
\begin{align*}
    \int_0^\infty \beta \delta_n(\beta (1+\beta)^{-1} ) d\beta &\leq 2m \int_0^{\beta^*} \beta \exp( -n^2 \beta^2 (1+\beta)^{-2} /(Cm)) d\beta + \int_{\beta^*}^\infty \beta (\tilde{C}(1+\beta)^{-1})^{\nu n / 6} d\beta~, \\
    &\leq 2m \int_0^{\beta^*} \beta \exp( -n^2 \beta^2 (1+\beta^*)^{-2} /(Cm)) d\beta + \frac{\tilde{C}^{\nu n / 6}}{\nu n / 6 -2 } (1+\beta^*)^{-\nu n / 6 + 2} \\
    &\leq 2 \sqrt{\pi} C (1+\beta^*)^2 \frac{m^{3}}{n^2} + \frac{\tilde{C}^{\nu n / 6}}{\nu n / 6 -2 } (1+\beta^*)^{-\nu n / 6 + 2}~. 
    %     \int_0^\infty \delta_n(\beta (1+\beta)^{-1} ) d\beta &\leq 2m \int_0^{\beta^*} \exp( -n^2 \beta^2 (1+\beta)^{-2} /(Cm)) d\beta + \int_{\beta^*}^\infty (\tilde{C}(1+\beta)^{-1})^{\nu n / 6} d\beta~, \\
    % &\leq 2m \int_0^{\beta^*} \exp( -n^2 \beta^2 (1+\beta^*)^{-2} /(Cm)) d\beta + \frac{\tilde{C}^{\nu n / 6}}{\nu n / 6 -1 } (1+\beta^*)^{-\nu n / 6 + 1} \\
    % &\leq 2 \sqrt{\pi} \sqrt{C} (1+\beta^*) \frac{m^{3/2}}{n} + \frac{\tilde{C}^{\nu n / 6}}{\nu n / 6 -1 } (1+\beta^*)^{-\nu n / 6 + 1}~. 
\end{align*}
Picking $\beta^* = \tilde{C}$ above gives 
\begin{equation}
\label{eq:d12}
    \int_0^\infty \beta \delta_n(\beta (1+\beta)^{-1} ) d\beta \leq \frac{\bar{C}}{n} + O\left( \frac{m^3}{n^2}\right)~,
\end{equation}
where $\bar{C}$ only depends on $\nu, C, \tilde{C}$. 
From (\ref{eq:d11}) and (\ref{eq:d12}) we conclude that
\begin{equation}
\label{eq:d13}
\mathbb{E} \| \hat{H}_n^{-1} - H^{-1} \|^2 = O\left( \frac{\| H^{-2} \|}{n}\right)~,
\end{equation}
proving (\ref{eq:ups2}).

Let us now bound the second term in the RHS of (\ref{eq:basic2}). 
We have 
$$\| \hat{H}^{-1} \upsilon\|^2 \leq \| \hat{H}^{-2} \| \|\upsilon\|^2~,$$
so by Cauchy-Schwartz we obtain
\begin{equation}
\label{eq:blop0}
\mathbb{E}[\| \hat{H}^{-1} \upsilon\|^2] \leq \left(\mathbb{E}[\| \hat{H}^{-4} \|] \right)^{1/2} \left(\mathbb{E}[\|\upsilon\|^4] \right)^{1/2}~.
\end{equation}
  By assumption, we have  
\begin{equation}
\label{eq:blop1}
\left(\mathbb{E}[\|\upsilon\|^4]\right)^{1/2} \leq K_Y \mathbb{E} [\| \upsilon\|^2] = \frac{K_Y}{n} \mathbb{E}[\|Y\|^2]~.
\end{equation}
Finally, we use the following lemma, showing that $\mathbb{E}[\| \hat{H}^{-4} \|]$ is bounded.
\begin{lemma}[Finite Second and Fourth Moments of $\hat{H}^{-1}$]
\label{lem:Hinv_bounded}
Assume $n > 24 / \nu$. Then 
\begin{equation}
\label{eq:hinvb}
\mathbb{E} [\| \hat{H}^{-2} \|] \leq \tilde{C}_2 \|H^{-1}\|^2~ \text{and }~\mathbb{E} [\| \hat{H}^{-4} \|] \leq \tilde{C}_4 \|H^{-1}\|^4~.
\end{equation}
\end{lemma}
From (\ref{eq:blop0}), (\ref{eq:blop1}) and (\ref{eq:hinvb}) we obtain 
\begin{equation}
    \mathbb{E}[\| \hat{H}^{-1} \upsilon\|^2] \leq \frac{ \xi \sqrt{\tilde{C}_4} \|H^{-1}\|^2 \mathbb{E}[\|Y\|^2]}{n}~,
\end{equation}
which, together with (\ref{eq:ups2}) yields 
\begin{equation}
    \label{eq:bigblu}
    \mathbb{E} \| \hat{\theta} - \theta^* \|^2 \leq O\left( \frac{\| H^{-1}\|^2 (\|\mathbb{E}[Y]\|^2 + K_Y \mathbb{E}[\|Y\|^2])}{n}\right)~,
\end{equation}
and therefore 
\begin{equation}
\label{eq:mainres}
    \mathbb{E}\ell(\hat{\theta}) - \ell(\theta^*) \leq O\left( \frac{\kappa \| H^{-1} \| (\|\mathbb{E}[Y]\|^2 + K_Y \mathbb{E}[\|Y\|^2]) }{n}\right)~,
\end{equation}
proving (\ref{eq:FIgen}) as claimed.

%a $n^{-1}$ convergence rate for the $\mathrm{FI}$ divergence.
%as claimed. 

Let us now control $\hat{{\bar\alpha}}$ such that $\hat{\theta} \in \Theta_{\hat{{\bar\alpha}}}$. 
From $\log p_{\theta}(\bar{x}|x) = \theta^\top \bar{\Phi}(x,\bar{x})$ we directly obtain 
$$\nabla^2 \log p_{\hat{\theta}}(\bar{x}|x) = \nabla^2 \log p_{{\theta^*}}(\bar{x}|x) + \sum_{k=1}^m (\hat{\theta}_k - \theta^*_k) \nabla^2 \bar{\Phi}_k(\bar{x}|x)~,$$
and thus, for any $(\bar{x}, x)$, 
\begin{align}
\label{eq:mim0}
   \| \nabla^2 \log p_{\hat{\theta}}(\bar{x}|x) - \nabla^2 \log p_{{\theta^*}}(\bar{x}|x) \| &\leq \sum_k |\hat{\theta}_k - \theta^*_k| \| \nabla^2 \bar{\Phi}_k(\bar{x}|x) \| \nonumber \\
    &\leq  \| \hat{\theta} - \theta^* \| \| \nabla^2 \bar{\Phi}(\bar{x}|x) \|~, \nonumber \\
    &\leq  \| \hat{\theta} - \theta^* \| \sqrt{m} M_{\bar{\Phi}} ~,
    %&\leq \frac{A \| H^{-1} \| \sqrt{m} M_{\bar{\Phi}}}{n} ~,
\end{align}
where $\| \nabla^2 \bar{\Phi}(\bar{x}|x) \|^2:= \sum_{k=1}^m \| \nabla^2 \bar{\Phi}_k(\bar{x}|x) \|^2$, and 
$M_{\bar{\Phi}} = \max_k \sup_{x,\bar{x}}  \| \nabla^2 \bar{\Phi}_k(\bar{x}|x) \| <\infty$ by assumption (ii).  
It follows from (\ref{eq:mim0}) that 
\begin{equation}
\label{eq:mim00}
    \inf_{(\bar{x}, x)} \lambda_{\text{min}}( \nabla^2 \log p_{\hat{\theta}}(\bar{x}, x)) \geq {\bar\alpha} - \| \hat{\theta} - \theta^* \| \sqrt{m} M_{\bar{\Phi}}~.
\end{equation}

We will now use tail probability bounds for the norm $\| \hat{\theta} - \theta^*\|$, captured in the following lemma:
\begin{lemma}[Tail bounds for $\|\hat{\theta} - \theta^*\|$]
\label{lem:tailbounds_dif}
We have 
\begin{equation}
\label{eq:brit}
\mathbb{P}( \| \hat{\theta} - \theta^* \| > t) \leq f_n(t/ \|H^{-1} \|), 
\end{equation}
with 
\begin{align}
f_n(s) & \leq \min\left[2m \exp\left( -n^2 \frac{\left( s / (2  \|g\|)\right)^2 }{ (1+ \left( s / (2  \|g\|)\right))^2 Cm  } \right), (\tilde{C} (2 C_Y)s^{-1} )^{\nu n / 6} \right]+ \\ 
& + 2m \exp( -n^2 \left( s / (2  \|g\|)\right)^2 / Cm) + \left( \frac{C_0}{s\sqrt{n}} \right)^{\nu n / 6} ~,\\
\end{align}
where $\tilde{C}, C, C_Y, \|g\|, \nu$ are constants from Assumptions \ref{ass:smallball}, \ref{ass:normbound_our}.   
Moreover, for $s \ll 1$, we have 
\begin{equation}
\label{eq:criq}
f_n(s) = \exp( -O(n (\log n + \log s)))
\end{equation}
\end{lemma}

From (\ref{eq:gruy}) and (\ref{eq:mim00}) we obtain
\begin{align*}
    \mathbb{E}_x \mathrm{KL}( p \,\|\, p_{\hat{\theta}}) &\leq \frac{1}{2\hat{{\bar\alpha}}}\left(\ell(\theta^*) + \| \hat{\theta} - \theta^*\|^2 \|H\| \right) \\
    &\leq \frac{\ell(\theta^*) + \| \hat{\theta} - \theta^*\|^2 \|H\|}{{\bar\alpha} - \| \hat{\theta} - \theta^* \| \sqrt{m} M_{\bar{\Phi}}}~,
\end{align*}
and therefore 
\begin{align*}
    \mathbb{P} \left[ \bar{\epsilon}^L \leq \frac{\ell(\theta^*)}{{\bar\alpha}}\left(1+\frac{bt + \ell(\theta^*)^{-1} \|H\| t^2}{{\bar\alpha} - bt}\right) \right] & \geq \mathbb{P}[ \| \hat{\theta} - \theta^*\| \leq t] \\
    &\geq 1-f_n(t/\|H^{-1}\|) ~,
\end{align*}
where $b = \sqrt{m} M_{\bar{\Phi}}$. 
As a result, for $t \ll \sqrt{m} M_\Phi \ell^\star \|H\|^{-1}$ we have 
\begin{align}
\mathbb{P} \left[ \bar{\epsilon}^L \leq \frac{\ell(\theta^*)}{{\bar\alpha}}\left(1+t\frac{b}{{\bar\alpha}}\right) \right] &\geq 1 - 4m \exp\left( -\frac{C n^2 t^2}{m \|H^{-1} \|^2}\right) + \left( \frac{C_0 \|H^{-1}\| }{t \sqrt{n}} \right)^{\nu n /6} \\
& = 1 - \exp( - O(n ( \log t + \log n - \log \sqrt{m}))) ~,
\end{align}
proving (\ref{eq:KLgen}).

%\comment{

Finally, let us prove (\ref{eq:scprop}). 
From (\ref{eq:deftheta}) we have 
\begin{equation}
\label{eq:bip}
\| \hat{\theta} - \theta^* \| \leq \| \Upsilon\| \|g\| + \| \hat{H}^{-1} \| \|\upsilon\|~.    
\end{equation}
The same argument leading to (\ref{eq:d13}) can be now applied to the first moment $\mathbb{E} \| \Upsilon \|$, 
yielding
\begin{align}
\label{eq:d11bis}
\int_0^1 \eta_n( \beta/(1-\beta)) d\beta &= 2m \int_0^1 \exp( -n^2 \beta^2 (1-\beta)^{-2} /(Cm)) d\beta \nonumber \\
&\leq 2m \int_0^1 \exp( -n^2 \beta^2 /(Cm)) d\beta \nonumber \\
&\leq  \sqrt{2\pi C} \frac{m^{3/2}}{n} ~, \text{ and }
%\int_0^1 \beta \eta_n( \beta/(1-\beta)) d\beta &= 2m \int_0^1 \beta \exp( -n^2 \beta^2 (1-\beta)^{-2} /(Cm)) d\beta \nonumber \\
%&\leq 2m \int_0^1 \beta \exp( -n^2 \beta^2 /(Cm)) d\beta \nonumber \\
%& \leq 2m \sqrt{\pi}\frac{\sqrt{C m/2}}{n} \mathbb{E}_{Z \sim \mathcal{N}(0, Cm/(2n^2))}[|Z|] \\
%&\leq  \widetilde{C} \frac{m^3}{n^2} ~.
\end{align}
\begin{align*}
   % \int_0^\infty \beta \delta_n(\beta (1+\beta)^{-1} ) d\beta &\leq 2m \int_0^{\beta^*} \beta \exp( -n^2 \beta^2 (1+\beta)^{-2} /(Cm)) d\beta + \int_{\beta^*}^\infty \beta (\tilde{C}(1+\beta)^{-1})^{\nu n / 6} d\beta~, \\
   % &\leq 2m \int_0^{\beta^*} \beta \exp( -n^2 \beta^2 (1+\beta^*)^{-2} /(Cm)) d\beta + \frac{\tilde{C}^{\nu n / 6}}{\nu n / 6 -2 } (1+\beta^*)^{-\nu n / 6 + 2} \\
   % &\leq 2 \sqrt{\pi} C (1+\beta^*)^2 \frac{m^{3}}{n^2} + \frac{\tilde{C}^{\nu n / 6}}{\nu n / 6 -2 } (1+\beta^*)^{-\nu n / 6 + 2}~. 
         \int_0^\infty \delta_n(\beta (1+\beta)^{-1} ) d\beta &\leq 2m \int_0^{\beta^*} \exp( -n^2 \beta^2 (1+\beta)^{-2} /(Cm)) d\beta + \int_{\beta^*}^\infty (\tilde{C}(1+\beta)^{-1})^{\nu n / 6} d\beta~, \\
     &\leq 2m \int_0^{\beta^*} \exp( -n^2 \beta^2 (1+\beta^*)^{-2} /(Cm)) d\beta + \frac{\tilde{C}^{\nu n / 6}}{\nu n / 6 -1 } (1+\beta^*)^{-\nu n / 6 + 1} \\
     &\leq 2 \sqrt{\pi} \sqrt{C} (1+\beta^*) \frac{m^{3/2}}{n} + \frac{\tilde{C}^{\nu n / 6}}{\nu n / 6 -1 } (1+\beta^*)^{-\nu n / 6 + 1}~. 
\end{align*}
Picking again $\beta^* = \tilde{C}$ above gives 
\begin{equation}
\label{eq:d12bis}
    \int_0^\infty \delta_n(\beta (1+\beta)^{-1} ) d\beta \leq \frac{\bar{C} m^{3/2} }{n} ~,
\end{equation}
and therefore 
\begin{equation}
    \label{eq:d14}
\mathbb{E} \| \Upsilon \| = O\left( \frac{\| H^{-1} m^{3/2}\|}{n}\right)~.    
\end{equation}
From (\ref{eq:bip}), using (\ref{eq:d14}) and again Cauchy-Schwartz, we obtain
\begin{align}
\label{eq:d15}
    \mathbb{E} \| \hat{\theta} - \theta^* \| &\leq \mathbb{E}[\|\Upsilon\|] \|g\| + \frac{\sqrt{\mathbb{E}[\| \hat{H}^{-2} \|] \mathbb{E}[\|Y\|^2]}}{\sqrt{n}} \nonumber \\
    &= O\left(\| H^{-1} \| \sqrt{\frac{\mathbb{E}[\|Y\|^2]}{n}} \right)~,
\end{align}
proving (\ref{eq:scprop}).

% Combining (\ref{eq:d15}) with (\ref{eq:mim00}) we finally deduce that 
% %,and the assumption that $\nabla^2 \log p_{\theta^*}(\bar{x} | x) \succeq \eta I$ for all $(\bar{x},x)$, we obtain  
% \begin{equation}
% \label{eq:bleble}
% \hat{\alpha}=\mathbb{E}\left[\inf_{(\bar{x}, x)} \lambda_{\text{min}}\left(\nabla^2 \log p_{\hat{\theta}}(\bar{x},x))\right) \right]\geq \alpha - O\left(\| H^{-1} \| \sqrt{\frac{\mathbb{E}[\|Y\|^2] m}{n}} \right)~,
% \end{equation}
% as claimed. 

% Finally, from (\ref{eq:mainres}) and the lower bound on $\hat{\alpha}$ from (\ref{eq:bleble}), we can apply the Barky-Emery criterion from eq. (\ref{eq:bakry}), giving (\ref{eq:klgen}). 

\end{proof}

\begin{proof}[Proof of Lemma \ref{lem:tails_to_exp}] 
Using a crude union bound, we have 
\begin{equation}
\label{eq:d1}
(1-t) \Sigma \preceq \hat{\Sigma}_n \preceq (1+s) \Sigma  
\end{equation}
with probability greater than $1 - \delta_n(t) - \eta_n(s)$. 
Under the event (\ref{eq:d1}), we equivalently have 
$$(1+s)^{-1} \Sigma^{-1} \preceq \hat{\Sigma}_n^{-1} \preceq (1-t)^{-1} \Sigma^{-1}~,$$
and hence 
$$\| \hat{\Sigma}_n^{-1} - \Sigma^{-1} \| \leq \| \Sigma^{-1}\| \max( |1-(1+s)^{-1}|, |1-(1-t)^{-1} | )~.$$
Denoting $Z = \| \hat{\Sigma}_n^{-1} - \Sigma^{-1}\|$, we thus have
\begin{align}
\label{eq:tailinvdif}
%\mathbb{P}( Z \leq \beta)  & \geq \mathbb{P}\left( Z \leq \|\Sigma^{-1}\| \max( |1-(1+s_\beta)^{-1}|, |1-(1-t_\beta)^{-1} | )\right) \\
\mathbb{P}( Z \leq \| \Sigma^{-1}\| \beta)  & \geq \mathbb{P}\left((1-t_\beta) \Sigma \preceq \hat{\Sigma}_n \preceq (1+s_\beta) \Sigma  \right)\\
&\geq 1 - \delta_n(t_\beta) - \eta_n(s_\beta)~,
\end{align}
where $s_\beta, t_\beta$ are defined such that 
$$ |1-(1+s_\beta)^{-1}| = \beta~,~ |1-(1-t_\beta)^{-1}| = \beta~.$$
We thus obtain $s_\beta = \frac{\beta}{1-\beta}$ for $\beta\in (0,1)$, and $t_\beta = \frac{\beta}{1+\beta}$ for $\beta\in (0,\infty)$. 
For a non-negative random variable $Z$ with c.d.f. $F(\beta) = \mathbb{P}( Z \leq \beta)$ we have
$$\mathbb{E}Z^2 = \int_0^{\infty} \beta^2 F'(\beta) d\beta = \int_0^\infty \beta (1-F(\beta)) d\beta~,$$
and therefore  
\begin{align*}
\mathbb{E} Z^2 & = \int_0^\infty \beta (1- F(\beta)) d\beta \\
&= \| \Sigma^{-2} \| \int_0^\infty \beta (1 - F(\|\Sigma\|^{-1} \beta) ) d\beta \\
&\leq \|\Sigma^{-2}\| \left(\int_0^\infty \beta \delta_n( \beta/(1+\beta)) d\beta + \int_0^1 \beta \eta_n( \beta/(1-\beta)) d\beta \right)~.
\end{align*}
\end{proof}

\begin{proof}[Proof of Lemma \ref{lem:Hinv_bounded}]
By directly applying Theorem \ref{thm:invertible_cov}, we have 
\begin{equation}
\label{eq:vuk}
\mathbb{P}( \| \hat{H}_n^{-1} \| \leq t^{-1} \| H^{-1}\|) \geq 1 - (\tilde{C} t)^{\nu n / 6}~.
\end{equation}
If $F(\beta) = \mathbb{P}( \| \hat{H}_n^{-1} \| \leq \beta)$, 
it follows that 
\begin{align*}
\mathbb{E}[ \| \hat{H}^{-4} \| ] &= \int_0^\infty \beta^4 F'(\beta) d\beta = 4 \int_0^\infty \beta^3 (1- F(\beta)) d\beta \\
&\leq 4 \int_0^\infty \beta^3 \min(1,(\tilde{C} \| H^{-1}\| \beta^{-1})^{\nu n /6}) d\beta \\
&= 4 \| H^{-1}\|^4 \tilde{C}^4 \int_0^\infty \min(1,\beta^{3-\nu n /6}) d\beta \\
& = \tilde{C}_4 \| H^{-1}\|^4~,
\end{align*}
where we used $\nu n /6   > 4$ in the last step. 
The second moment is treated analogously. 
\end{proof}

\begin{proof}[Proof of Lemma \ref{lem:tailbounds_dif}]
As we argued previously, from (\ref{eq:deftheta}) we have that 
$$\| \hat{\theta} - \theta^* \| \leq \| \Upsilon \| \|g\| + \| \hat{H}^{-1} \| \|\upsilon\| ~.$$
We will use tail bounds for $\| \Upsilon\|$, $\|\hat{H}^{-1} \|$ and $\| \upsilon\|$ and combine them with a crude union bound to yield the desired tail control. 
Recall from eq (\ref{eq:tailinvdif}) that  
\begin{equation}
\label{eq:tailinvdif2}
%\mathbb{P}( Z \leq \beta)  & \geq \mathbb{P}\left( Z \leq \|\Sigma^{-1}\| \max( |1-(1+s_\beta)^{-1}|, |1-(1-t_\beta)^{-1} | )\right) \\
\mathbb{P}( \| \Upsilon\| \leq \| H^{-1}\| t)  \geq 1 - \gamma_n(t)~,
\end{equation}
where
\begin{equation}
\gamma_n(t) = 
\begin{cases}
\delta_n\left(\frac{t}{1+t}\right) + \eta_n\left(\frac{t}{1-t}\right)~, & \text{ if } t \leq 1\,,\\
 \delta_n\left(\frac{t}{1+t}\right) & \text{ otherwise,}
\end{cases}
\end{equation}
% and 
% \begin{align}
% \label{eq:tailinvdif3}
% %\mathbb{P}( Z \leq \beta)  & \geq \mathbb{P}\left( Z \leq \|\Sigma^{-1}\| \max( |1-(1+s_\beta)^{-1}|, |1-(1-t_\beta)^{-1} | )\right) \\
% \mathbb{P}( \| \Upsilon\| \leq \| \Sigma\|^{-1} \beta)  & \geq 1 - \delta_n\left(\frac{\beta}{1+\beta}\right) 
% \end{align}
% for $\beta >1$. 
with
\begin{equation}
\delta_n(s) = \min((\tilde{C}(1-s))^{\nu n/6}, 2m \exp( -n^2 s^2 / C m))~,~\eta_n(s) = 2m \exp( -n^2 s^2 / C m)~.
\end{equation}
We also obtained in (\ref{eq:vuk}) 
\begin{equation}
\label{eq:vuk2}
\mathbb{P}( \| \hat{H}_n^{-1} \| \leq t \|H^{-1}\|) \geq 1 - \tilde{\gamma}_n(t)~,
\end{equation}
with 
\begin{equation}
\tilde{\gamma}_n(t) = \min(1,(\tilde{C}  t^{-1})^{\nu n / 6})~,
\end{equation}
and by Assumption \ref{ass:normbound_our} we know that $\|\upsilon\| \leq \frac{K_Y \sqrt{\mathbb{E}[\|Y\|^2]}}{\sqrt{n}}$ almost surely. 
Therefore, via a union bound we obtain 
\begin{align}
\mathbb{P}( \| \hat{\theta} - \theta^* \| \leq \|H^{-1} \| t) &\geq \mathbb{P}\left[ \max\left( \| \Upsilon\| \|g\| , \|\hat{H}^{-1 }\| K_Y \sqrt{\mathbb{E}[\|Y\|^2]/n} \right) \leq \|H^{-1}\| t/2\right] \\
&\geq 1 - \gamma_n( t / (2 \|g\|)) - \tilde{\gamma}_n( \sqrt{n} t / (2 C_Y) )~,
\end{align}
and hence $\mathbb{P}( \| \hat{\theta} - \theta^* \| > s) \leq f_n\left(\frac{s}{ \|H^{-1}\|}\right)$ with 
\begin{align*}
f_n(s) &= \gamma_n\left( s / (2 \| \|g\|)\right) + \tilde{\gamma}_n\left( \sqrt{n} s / (2 C_Y) \right)~.
\end{align*}
Finally, we verify that
\begin{align*}
f_n(s)&= \gamma_n\left( s / (2  \|g\|)\right) + \min(1, (\tilde{C} (2 C_Y)s^{-1} n^{-1/2} )^{\nu n / 6}) \\
&= \gamma_n\left( s / (2  \|g\|)\right) + \left( \frac{C_0}{s\sqrt{n}} \right)^{\nu n / 6} \\
&\leq \min\left[2m \exp\left( -n^2 \frac{\left( s / (2  \|g\|)\right)^2 }{ (1+ \left( s / (2  \|g\|)\right))^2 Cm  } \right), (\tilde{C} (2 C_Y)s^{-1} )^{\nu n / 6} \right]+ \\ 
& + 2m \exp( -n^2 \left( s / (2  \|g\|)\right)^2 / Cm) + \left( \frac{C_0}{s\sqrt{n}} \right)^{\nu n / 6} \\
&= \min\left[\exp\left( - \frac{ n^2 C_1^2 s^2  }{ (1+ C_1 s )^2 Cm  } + \log(2m)\right) , (C_0 s^{-1})^{\nu n /6}\right]
+\exp\left( - \frac{ n^2 C_1^2 s^2  }{ Cm  } + \log(2m)\right)+
%& 2m \exp( -n^2 \left( s / (2 \|H^{-1} \| \|g\|)\right)^2 / Cm) + \\
 \left( \frac{C_0}{s\sqrt{n}} \right)^{\nu n / 6} ~.
\end{align*}
Finally, we verify that if $\log s \ll 1$, the last term dominates as $n$ increases, showing (\ref{eq:criq}). 

% where $s_\beta, t_\beta$ are defined such that 
% $$ |1-(1+s_\beta)^{-1}| = \beta~,~ |1-(1-t_\beta)^{-1}| = \beta~.$$
% We thus obtain $s_\beta = \frac{\beta}{1-\beta}$ for $\beta\in (0,1)$, and $t_\beta = \frac{\beta}{1+\beta}$ for $\beta\in (0,\infty)$. 

\end{proof}

%% file: proofs_scalar_potential.tex
\section{Proof of \Cref{th:conditionally-log-concave-scalar-potential}}
\label{sec:proof_th_cslc}

We directly compute the Hessian
\begin{align*}
    -\nabla^2_{\bar x_1} \log p(\bar x_1 | x_1)
    &= -\bar G_1 \nabla^2_{x} \log p(x) \bar G_1\trans \\
    &= \bar G_1 \paren{K - \mathrm{diag}\paren{(v''(x[i]))}_i} \bar G_1\trans,
\end{align*}
where we have used
\begin{align*}
    p(\bar x_1 | x_1) &= \frac{p(x)}{p(x_1)}.
\end{align*}

Both terms in the Hessian can now be bounded from below. The assumption on the range of $\bar G_1$ implies that
\begin{align*}
    \bar G_1 K \bar G_1\trans \succeq \lambda |\omega_0|^\eta \Id,
\end{align*}
and the assumption on $v''$ implies that
\begin{align*}
    \bar G_1 \mathrm{diag}\paren{(v''(x[i]))}_i \bar G_1\trans \succeq -\gamma \bar G_1 \bar G_1\trans = -\gamma \Id,
\end{align*}
where we have used the fact that $\bar G_1$ is an orthogonal projector.

Combining the two then gives
\begin{align*}
    -\nabla^2_{\bar x_1} \log p(\bar x_1 | x_1) &\succeq \paren{\lambda |\omega_0|^\eta - \gamma} \Id,
\end{align*}
and the assumption on $|\omega_0|$ guarantees that $\lambda |\omega_0|^\eta - \gamma > 0$. Similarly, the assumption $v'' \leq \delta$ implies that
\begin{align*}
    -\nabla^2_{\bar x_1} \log p(\bar x_1 | x_1) &\preceq \paren{\lambda \Omega^\eta + \delta} \Id,
\end{align*}
where $\Omega = \sup |\omega|$ is the maximum frequency, which concludes the proof.

%LEARNING